\documentclass{article}
\usepackage[margin=1.5in]{geometry}

\usepackage[square,numbers]{natbib}
\usepackage{amssymb}

\usepackage[utf8]{inputenc} 
\usepackage{url}            
\usepackage{booktabs}       
\usepackage{amsfonts}       
\usepackage{nicefrac}       
\usepackage{microtype}      
\usepackage{hyperref}       
\usepackage{xcolor}         

\usepackage{wrapfig}
\usepackage{amsthm}
\usepackage{amsmath}

\usepackage{subfigure} 
\usepackage{enumitem}
\usepackage{natbib}

\usepackage{algorithm}
\usepackage{algorithmic}

\usepackage{thmtools, thm-restate} 

\newtheorem{theorem}{Theorem}
\newtheorem{lemma}{Lemma}
\newtheorem{corollary}{Corollary}
\newtheorem{assumption}{Assumption}

\theoremstyle{remark}
\newtheorem{remark}{Remark}

\usepackage{mathtools}

\author{Liran Szlak, Ohad Shamir}

\title{Convergence Results For Q-Learning With Experience Replay}

\begin{document}

\maketitle

\begin{abstract}
A commonly used heuristic in RL is experience replay (e.g.~\citet{lin1993reinforcement, mnih2015human}), in which a learner stores and re-uses past trajectories as if they were sampled online. In this work, we initiate a rigorous study of this heuristic in the setting of tabular Q-learning. We provide a convergence rate guarantee, and discuss how it compares to the convergence of Q-learning depending on important parameters such as the frequency and number of replay iterations. We also provide theoretical evidence showing when we might expect this heuristic to strictly improve performance, by introducing and analyzing a simple class of MDPs. Finally, we provide some experiments to support our theoretical findings. 
\end{abstract}

\section{Introduction}
Q-learning is a well-known and commonly used algorithm for reinforcement learning. In recent years, a technique referred to as experience replay \citep{lin1993reinforcement, mnih2015human} has been suggested as a mechanism to improve Q-learning by allowing the learner to access previous experiences, and use them offline as if they were examples currently sampled from the world. It has been suggested that using past experiences in such a way might allow Q-learning to better converge to the optimal Q values, by breaking the time and space correlation structure of experiences as they are sampled from the real world, allowing for policy updates not dependent on the current time and location in the markov decision process. Moreover, using experience replay improves the efficiency of data usage, since every experience is used for learning more than once. This can be useful in situations where data acquirement is costly or difficult.
Experience replay draws inspiration from a biological phenomenon- in animals and humans, replay events occur following a learning session. \citet{momennejad2017offline} demonstrate in an fMRI study that experience replay occurs in short breaks during learning, and that replay correlates to performance in the task. Thus, establishing a theoretical understanding of this phenomenon is of interest to both the machine learning and the neuroscience communities.

Experience replay can be incorporated into various RL algorithms, and various methods for sampling from the memory buffer have been suggested and experimentally explored in different tasks. \citet{schaul2015prioritized} and \citet{wang2016sample} use prioritized experience replay, prioritizing experiences which are important (in some sense) over others, to get remarkable results on the well-known Atari benchmark. 
\citet{liu2017effects} use uniformly sampled experiences and show that the memory buffer size is a critical hyper-parameter, where a buffer too small doesn't improve the convergence rate of Q-learning, and a buffer too large causes over-fitting and deteriorates performance. 
\citet{zhang2017deeper} experimentally demonstrate for both tabular and linear function approximation Q-learning that combining online experience with replayed experience in every iteration allows for faster convergence and stable learning under different memory buffer sizes.
\citet{wang2019boosting} propose a replay scheme where recent experiences have a higher probability to be chosen, emphasizing recent experiences over older ones.
Unfortunately, not much work has been done on theoretically analyzing the heuristic of experience learning. An interesting theoretical work by \citet{vanseijen2015deeper} shows the relation between learning and planning, and formulates experience replay as a planning step as in the famous Dyna algorithm.
However, a thorough investigation of the theoretical origins for its advantages and disadvantages in various settings is still in need. Here, we attempt to make a first step in this direction.

In order to do so, in this work we focus on tabular Q-learning and uniform sampling from a replay buffer. We show that Q-learning with uniform experience replay converges to the optimal Q values, and analyze its convergence rate. Our results indicate that the convergence rate of Q-learning with experience replay is comparable to that of vanilla Q-learning, suggesting that experiences sampled from the memory buffer can replace learning iterations in the real world, although not necessarily speeding up convergence in terms of the total amount of iterations needed for convergence. Our bound also indicates that using experience replay too extensively may result in a significantly higher total number of iterations compared to the number of iterations required without replay. This is explained by the error accumulated by overfitting to the experiences in the buffer, and the need to continually correct it with online experiences. However, maybe surprisingly, we get that as long as we continue to sample from the real world at a sufficient rate, even extensive experience replay will not cause divergence. Our theoretical findings correspond to recent empirical findings showing how the replay ratio (number of replay updates per environment step) affects performance \citep{fu2019diagnosing, fedus2020revisiting}.

We then continue to present a specific scenario in which experience replay strictly and provably improves performance compared to standard Q-learning. Specifically, we define and discuss a class of MDPs, where some critical transitions in the environment are rare with respect to the time horizon. In those environments, which draw inspiration from biological phenomenon observed in humans \citep{momennejad2017offline}, we prove that Q-learning will not converge to the optimal policy, while running experience replay after the end of the online learning phase will allow convergence to the optimal Q values.
Finally, we provide experiments complementing our theoretical findings and discusses implications of this work.  
\section{Algorithm and setting}
\label{sec:algorithmSetting}
\paragraph{Setting and notations}
A reinforcement learning problem can be formulated as a Markov decision process (MDP), $(S,A,P,R)$, where $S$ is a set of states, $A$ is a set of actions, $P$ is a transition probability: $\forall s,s',a : p_{s,s'}(a) := p(s'|s,a)$. For all state-action pairs, $\{r_t(s,a)\}_{t = 1}^T$ is a bounded reward random variable with $E[r_t(s,a)] = R(s,a)$, $\forall s,a: |r_t(s,a)| \leq R_{max}$, and rewards are sampled i.i.d over all rounds per state-action pair.
At every time step $t = 0,1,2,...$, the learning agent perceives the current state of the environment $s_t$, and chooses an action $a_t$ from the set of possible actions at that state. The environment then stochastically moves to a new state $s_{t+1}$ according to the transition probability $P$, and sends a numerical reward signal $r_t(s_t, a_t)$.
The learner's objective is to learn a policy $\pi$, that maps states $s \in S$ to actions $a \in A$, such that it maximizes the expected discounted future reward from each state $s$, i.e. maximize $V^{\pi} \left( s \right) = \mathbb{E} \left[ \sum_{i=0}^\infty \gamma^i \cdot r_i | s_0,  \pi \right]$, where $\gamma \in (0,1)$ is the discount factor which discounts the value of rewards in the future.
\paragraph{Algorithm}
We consider the well-known Q-learning algorithm, where we add $M$ iterations of experience replay every $K$ iterations. We focus here on the tabular setting and uniform sampling from the memory buffer. For clear reasons, tabular Q-learning is unsuitable for large or continuous environments, however, even in the tabular setting, it is not a-priori clear how experience replay would affect convergence. On the one hand, replay allows for faster propagation of information throughout the Q-values table, as well as for increased data efficiency and thus should improve convergence rate. On the other hand, sampling from a replay buffer introduces additional error due to sampling from an empirical distribution of experiences rather than from the real distribution. Here, we analyze the convergence rate of Q-learning with experience replay to understand how these factors affect convergence, compared to regular Q-learning. We note that for simplicity, we do not bound the size of the memory buffer. However, this is a necessary first step in performing a rigorous analysis of the effects of replay on the convergence rate, and may be alleviated by ensuring a sufficient representation of the state-action space in the buffer or by controlling the transition distribution in the buffer over time. The algorithm proceeds as follows: at every iteration $t$, the learner perceives $s_t$, chooses an action $a_t$, receives a reward $r_t(s_t,a_t)$ and transitions to a new state $s_{t+1}$. Then $Q(s_t,a_t)$ is updated: $Q_t(s_t,a_t) = (1-\alpha_t(s_t,a_t)) \cdot Q_{t-1}(s_t,a_t) + \alpha_t(s_t,a_t) \cdot \left( r_t + \gamma \cdot max_{a'} Q_{t-1}(s_{t+1},a') \right)$, and the transition $(s_t,a_t,s_{t+1},r_t)$ is stored in the memory buffer. Every $K$ iterations, $M$ iterations of replay are performed: a uniformly sampled memory $(s_1,a,s_2,r)$ is sampled from the buffer, and $Q(s_1,a)$ is updated using the regular update of Q-learning. An explicit algorithm is specified in appendix~\ref{app:alg}.
\section{Convergence rate of Q-Learning with experience replay}
\label{sec:convergenceRateProofSec}
We now turn to analyze the rate of convergence of the algorithm.
Here, we are showing a finite-time convergence result. We also give an asymptotic convergence result with a simpler proof in appendix~\ref{app:asymptoticProof}. 
We give results for two settings, a simplified 'synchronous' setting which allows us to analyse the convergence rate in a simple, albeit non realistic setting, and the standard 'asynchronous' setting, for which we derive results as a corollary from the simple synchronous case. In the 'synchronous Q-learning' setting~\citep{even2003learning}, all state-action pairs are being updated at every iteration, simultaneously. In this setting, at every iteration we sample for each $(s,a)$ pair a reward and a transition to a new state, and update the Q-values of all pairs simultaneously, using the Q values of the previous iteration. Sampling from the replay buffer is done by uniformly sampling an experience of each pair (s,a) from all of its memories stored in the buffer. The simplicity of the synchronous setting allows us to analyze the convergence rate of Q-learning with experience replay and easily derive from it an analysis for the asynchronous setting. The explicit algorithms of synchronous and asynchronous Q-learning with replay are found in appendix~\ref{app:synchronousQlearning}, \ref{app:alg}.
\paragraph{Convergence rate results}
In what follows we consider the learning rate $\alpha_t(s,a) = \frac{1}{n(s,a)}$ where $n(s,a)$ is the number of updates done on $Q(s,a)$ until time $t$ (including $t$), for all state-action pairs that are updated at iteration $t$, and $\alpha_t(s,a) = 0$ for the rest. 
We first give results for the simplified synchronous setting (Theorem~\ref{theorem:convergenceRate}), and then for the asynchronous setting (corollary~\ref{corollary:convergenceSingleUpdatePerStep}) by slight modifications to the former analysis. 
In the following statements and throughout the paper we use the notations: $V_{max} := \sum_{i=0}^\infty \gamma^i \cdot R_{max} = \frac{R_{max}}{1-\gamma}$, and $D_0 := V_{max} + max(\|Q_0\|_\infty,V_{max})$, where $Q_0$ are the initial Q-values.
We assume the following assumption in our analysis:
\begin{assumption}
\label{assum:ConditionsForConvergence}
There exists some constant $c \in (0,1)$ s.t. from any initial state $s$, within $\frac{|S||A|}{c}$ iterations, all state-action pairs are visited.
\end{assumption}
Note that assumption~\ref{assum:ConditionsForConvergence} requires the algorithm to have enough exploration in order to have enough samples of each state-action pair in the memory buffer. However, it does not require the algorithm to have a stationary distribution of the exploration strategy. Through assumption 1, we get that at every time $t$, for all $s \in S, a \in A$, the number of samples in the memory buffer of transitions from state $s$ with action $a$ is at least $\frac{c \cdot t}{|S| \cdot |A|}$. Importantly, $c$ captures the effect of experience replay on the bound, since $c$ itself is bounded by the parameters of the replay we perform. Specifically, $c \leq \frac{K}{M+K}$ (proof in appendix~\ref{app:bound_on_c}). Thus the effect of replay on convergence rate is captured in $c$.
\begin{remark}
\label{remark-assump}
Assumptions of this type are generally essential for the theoretical analysis of Q-learning. 
An assumption equivalent to assumption~\ref{assum:ConditionsForConvergence}, of a 'covering time', a constant time window in which all state-action pairs are visited, was also assumed in proving the convergence rate of Q-learning \citep{even2003learning}. Similar assumptions on a minimum state-action occupancy probability, mixing time, or covering time were assumed in \citep{li2020sample, beck2012error, qu2020finite}. \citet{liu2017effects} show that under-representation of state-action pairs in the buffer can lead to sub-optimal results when using experience replay. 
Hence, this assumption is critical and standard for achieving stable convergence to the optimal policy. We note that such an assumption is satisfied if we have a generative model from which we can choose a sampling distribution, or in highly mixing domains where $c$ is related to the notion of hitting time.
\end{remark}
Despite remark~\ref{remark-assump}, we note that assumption~\ref{assum:ConditionsForConvergence} is strict and difficult to prove that happens with high probability in arbitrary MDPs, as this depends on the markov chain underlying the decision process as well as the Q-values themselves during learning. Thus, we provide a probabilistic relaxation of this assumption (at the cost of a more involved bound): we show that if assumption~\ref{assum:ConditionsForConvergence} holds with probability $\frac{1}{2}$, then for any $\delta \in (0,1)$ with probability $\geq 1- \delta$, for all $t \in \{1,...T\}$ within $\frac{|S||A|}{c} \log_2(\frac{T}{\delta})$ iterations all state-action pairs are visited.  We use this new probabilistic assumption to bound the convergence rate in the synchronous setting, similarly to Theorem~\ref{theorem:convergenceRate}. Unfortunately, in the asynchronous setting, the stricter assumption is required (at least in the analysis we bring here). Notably, this is also the case for Q-learning without experience replay \citep{even2003learning}. The full details of the analysis and convergence rate under the relaxed assumption are found in appendix~\ref{app:assumption_1_relaxation}. 
\begin{theorem}
\label{theorem:convergenceRate}
Let $Q_T$ be the Q-values of synchronous Q-learning with experience replay after $T$ updates using a learning rate $\alpha_t(s,a) = \frac{1}{n(s,a)}$, where $n(s,a)$ is the number of updates done on $Q(s,a)$ until iteration $t$. Let $\epsilon_1 > 0$, $\delta \in (0,1)$, $\gamma \in (0,1)$.
For all $s \in S, a \in A, t$: let $|r_t(s,a)| \leq R_{max} < \infty$.
Then, under assumption~\ref{assum:ConditionsForConvergence}, with probability at least $1 - \delta$, $\| Q_t - Q^\ast\|_{\infty} \leq \epsilon_1$ for all $t \geq T$ with: 
$T = \tilde{\Omega} \left( 3^{\frac{2}{1-\gamma} log(\frac{D_0}{\epsilon_1})} \cdot \frac{|S| |A| \cdot R_{max}^2 \cdot \log(\frac{|S||A|}{\delta})  }{c \cdot (1-\gamma)^4 \epsilon_1^2 } \right)$
\end{theorem}
Here, the $\tilde{\Omega}$ notation hides constants and factors logarithmic in $\gamma$ and double logarithmic in $\epsilon_1, D_0$.\\
Theorem~\ref{theorem:convergenceRate} gives results for the synchronous setting, and we get the results for the asynchronous case in the following corollary by slight modifications to the proof of Theorem~\ref{theorem:convergenceRate}.
\begin{corollary}
\label{corollary:convergenceSingleUpdatePerStep}
Let $Q_T$ be the value of Q-learning with experience replay after $T$ iterations, using a learning rate $\alpha_t(s,a) = \frac{1}{n(s,a)}$ where $n(s,a)$ is the number of updates done on $Q(s,a)$ until time $t$, and $\alpha_t(s,a) = 0$ for state-action pairs that are not updated at iteration $t$. Let $\epsilon_1 > 0$, $\delta \in (0,1)$, $\gamma \in (0,1)$.
Then, under assumption~\ref{assum:ConditionsForConvergence}, for 
\begin{align*}
& T = \tilde{\Omega} \left( \left(\frac{3 |S||A|}{c} \right)^{\frac{2}{1-\gamma} log(\frac{D_0}{\epsilon_1})} \cdot \left( \frac{ |S| |A|  R_{max}^2 \log(\frac{|S||A|}{\delta}) }{c \cdot (1-\gamma)^4 \epsilon_1^2 } + log(\frac{|S||A|}{c^3}) \right) \right)
\end{align*}
$ \forall t \geq T: \| Q_t - Q^\ast\|_{\infty} \leq \epsilon_1 $ with probability at least $1 - \delta$.
\end{corollary}
It is important to emphasize the role of $c$ in the bounds we get in Theorem~\ref{theorem:convergenceRate} and corollary~\ref{corollary:convergenceSingleUpdatePerStep}. As mentioned, $c \leq \frac{K}{M+K}$ meaning that the more replay iterations we do (i.e $M \rightarrow \infty$) or the more frequently we perform replay (i.e $K \rightarrow 1$), the higher will be the number of iterations required for convergence. Thus, $c$ is a critical parameter in the bound, encapsulating the dependence between performance and the replay parameters. This stems from the fact that although replay saves us real-world iterations drawn from the MDP, it also introduces an additional error due to sampling from an empirical distribution rather than from the actual transition distribution. Thus, any replay scheme has to compensate for this error by having a sufficient sampling rate from the MDP.

The convergence rate we get is comparable to the convergence rate of Q-learning with a linear learning rate as given in \cite{even2003learning}, where the bound they give for synchronous Q-learning is $\Omega ( (2+ \Psi)^{\frac{2}{1-\gamma}\cdot log(\frac{V_{max}}{\epsilon_1})} \cdot \frac{R_{max}^2 log(\frac{2|S||A| R_{max}}{\delta (1-\gamma)^2 \Psi \epsilon_1})}{\Psi^2 (1-\gamma)^4 \epsilon_1^2})$ where $\Psi$ is some positive constant. For the asynchronous case, they give a bound of $\Omega ( (L+\Psi L+1)^{\frac{2}{1-\gamma}\cdot log(\frac{R_{max}}{\epsilon_1})} \cdot \frac{V_{max}^2 log(\frac{2|S||A| R_{max}}{\delta (1-\gamma)^2 \Psi \epsilon_1})}{\Psi^2 (1-\gamma)^4 \epsilon_1^2} )$, where $L$ indicates the 'covering time' constant - a time interval in which all state-action pairs are visited. 
The bounds we get are comparable to the ones in \cite{even2003learning} in the sense that the dependence on the parameters $\delta, \epsilon_1, \gamma, V_{max}$ are similar, but furthermore, our parameter $\frac{1}{c}$ can be viewed as an equivalence to the scaled covering time $\frac{L}{|S||A|}$ - $L$ represents the time it takes to cover all state and actions, whereas $c$ represents the fraction of transitions seen per state and action. When we perform replay, we reduce $c$, causing the number of iterations needed for convergence to grow, similarly to the effect caused by a larger covering time $L$. Indeed, the dependence on $c$ is much stronger, because we are not only accounting for the time it takes to update all state and action pairs, but also for the noise accumulation due to the finite number of samples in our memory buffer.   
The similarities in the bounds suggest that experience replay iterations can in some cases replace 'real-life' updates, and thus save us iterations in the real world. When we have enough memories, replaying is almost equivalent to sampling from the actual MDP. However, employing experience replay too extensively can lead to the opposite effect and highly increase the number of iterations needed for convergence, as reflected by the dependence on $c$ which in turn depends on the replay parameters $M,K$. As we perform M iterations of replay every K iterations in the real world, the requirement to have enough real samples of each state and action in the memory buffer at every iteration (assumption~\ref{assum:ConditionsForConvergence}) makes sure that one cannot simply run as many replay iterations as wanted, and that there must be an ongoing balance between real-world and replay iterations to ensure convergence, at least in the proof we provide here. These findings give theoretical support to experimental findings \citep{fu2019diagnosing, fedus2020revisiting} showing that algorithm performance is sensitive to the replay ratio - the number of replay updates per real world steps.

In this work, we use a linearly decreasing learning rate and thus compare our bounds to those of \citep{even2003learning} who use the same learning rate. However, it has been shown that a constant or a re-scaled linear learning rate lead to better convergence rate guarantees for Q-learning \citep{li2020sample, li2021q}. We conjecture that the effect of experience replay on the convergence rate of Q-learning with other learning rates will be similar to the effect we find here, however, a detailed analysis is deferred to future work.
\subsection{Proof sketches}
\label{sec:proofsSecInMain}
We now give proof sketches for our theorems. Full proofs are found in appendices~\ref{app:convergenceRateProof}, ~\ref{sec:corollaryProof}.
\paragraph{Synchronous setting}
We refer from this point on to the synchronous setting, i.e. every state-action pair is updated at every iteration, and thus $\alpha_t(s,a) = \frac{1}{t}$. We use the framework of iterative stochastic algorithms \citep{bertsekas1996neuro}, showing how Q-learning with experience replay can be written as an iterative stochastic algorithm (proof in Appendix~\ref{app:IterativeStocahsticAlg}), similar to the formulation in \cite{bertsekas1996neuro}.
We start by showing that the difference between the Q values and the optimal Q values, $Q(s,a)-Q^\ast(s,a)$, can be decomposed to two terms, $Y_t(s,a)$, the 'deterministic' term, which describes the process of convergence that would take place if all updates were done using the expected value of all random variables, and $W_t(s,a)$, the 'error' term, that holds the error accumulated by the stochastic sampling of the random variables. Using this decomposition, we show that we can bound the distance between the Q values and the optimal Q values by increasingly smaller constants, which leads to convergence. This is a common technique in iterative stochastic algorithms. However, since we are sampling both from the MDP and from a memory buffer, in our case the error accumulated by sampling of memories does not have an expected value of $0$, as opposed to the error accumulated by sampling from the MDP. This is the main challenge in proving the convergence of Q-learning with experience replay. 

Using this decomposition, we recursively bound the difference $Q_t(s,a)-Q^\ast(s,a)$ by a sequence of distance constants: $D_{k+1} = (\gamma + \epsilon)\cdot D_k$, for $\epsilon = \frac{1 - \gamma}{2}$. Clearly, $D_k \rightarrow 0$ as $k \rightarrow \infty$, thus this results in convergence to the optimal Q values. We prove that for each $k$, there exists a time $t_k$ such that for all iterations following $t_k$ our Q values are not further than $D_k$ from the optimal ones.
\begin{restatable}{lemma}{convergenceInLimit}
\label{lemma:convergenceInLimit}
Let $\gamma \in (0,1)$. For each $k$, $\exists t_k$ s.t. for all $t \geq t_k$, $\| Q_{t}(s,a) - Q^\ast(s,a) \| \leq D_k$.
\end{restatable}
Next, we compute what are the points $t_k$, which gives us the rate of convergence. 
In the proof of lemma~\ref{lemma:convergenceInLimit} we show that for every $s,a$ and $t \geq t_k$: 
$| Q_{t}(s,a) - Q^\ast(s,a) | \leq | Y_t(s,a) + W_{t:t_k}(s,a)|$
where $Y_t(s,a)$ holds the deterministic part of the error and $W_{t:t_k}(s,a)$ holds the stochastic part of the error, accumulated from iteration $t_k$ to $t$. 
Using this decomposition, we bound the number of iteration for $Y_t(s,a)$ to become smaller than $(\gamma + \frac{2 \epsilon}{3}) D_k$, and the number of iterations until $W_t(s,a)$ is smaller than $\frac{\epsilon}{3} D_k$. 
Together, this gives us the number of iteration to have $Y_t(s,a) + W_t(s,a) < D_{k+1}$.
\begin{restatable}{lemma}{deterministicPartConvergenceTime}
\label{lemma:deterministicPartConvergenceTime}
Let $\gamma \in (0,1)$, $\epsilon = \frac{1 - \gamma}{2}$, and assume at time $t_k$ we have $Y_{t_k}(s,a) = D_k$. Then, for all $t > 3 \cdot t_k$ we have $Y_t(s,a) < (\gamma + \frac{2\epsilon}{3})D_k$
\end{restatable}
The above lemma tells us what is the number of updates that needs to be done on a Q-value of a state-action pair if updates were done using the expected reward and the expected next state Q value. 
Next we compute at what time point do we have $W_t \leq \frac{\epsilon}{3} D_k$. We show that by setting an appropriate $t_0$, the noise accumulated from $t_k$ until $t$, for all $t \in [t_k, t_{k+1}]$, is bounded by $\frac{\epsilon}{3} D_k$, recursively, for all $k$. Combining this with lemma~\ref{lemma:deterministicPartConvergenceTime} gives us the required convergence rate.
In order to show our noise term is bounded, we separate online and replay iterations, and bound the noise contribution of each separately.
We emphasize the difference in the noise term between online and replay iterations- for online iterations (samples from the MDP) the expected value of the noise term in every iteration, $w$, is $0$. On the other hand, for replay iterations, the expected value is not zero, since in these iterations the expectation of $r_t(s,a)$ and $s_{t+1}$ are not equal to the true expectations as in the MDP, but to the empirical means of the experiences collected in the buffer. We bound the probability that for a time point $m \in [t_{k+1}, t_{k+2}]$, $W_m(s,a)$ is large, i.e.: $\Pr(|W_m(s,a)| \geq \frac{\epsilon}{3}D_k)$.
First we note that since we are using $\alpha_t = \frac{1}{t}$, $W_m$ is in fact an average of samples $w_i$ for $i\in \{1,\ldots,m\}$, where some are online samples from the MDP, and some are samples from the memory buffer.
We denote online samples as $q_i$, and replay samples as $g_i$. 
Since we are taking $M$ replay iterations every $K$ online iterations, we have that out of the $m$ samples in $W_m$, $\frac{mM}{M+K}$ of then are replay samples, and $\frac{mK}{M+K}$ of them are samples from the MDP. 
We omit in the following analysis the state-action pair, and refer to a single and constant pair.  
We write $W_m$ as two sums, each summing over samples of one kind:
\begin{align*}
\Pr\left(|W_m| \geq \frac{\epsilon}{3}D_k\right) &\leq \Pr \left( \left| \frac{1}{t_k+m} \sum_{i=1}^{\frac{mM}{M+K}}g_i \right| \geq \frac{\epsilon}{6}D_k \right) + \Pr \left( \left| \frac{1}{t_k+m} \sum_{i=1}^{\frac{mK}{M+K}}q_i \right| \geq \frac{\epsilon}{6}D_k \right)
\end{align*}
We bound each part separately. For the second term, relating to samples from the MDP, we can simply use Hoeffding's inequality. For the first term, relating to the samples from the replay buffer, we use assumption~\ref{assum:ConditionsForConvergence} and Azuma's inequality. Here, importantly, we use the fact that the number of samples of $(s,a)$ in the memory at time $t$ is at least $\frac{c \cdot t}{|S| \cdot |A|}$, to show that with exponentially high probability, the empirical mean in the buffer is close to the true mean, allowing us to bound the additional error.

To conclude the proof, we compute for what $k$ do we have that $D_k \leq \epsilon_1$, which is the desired precision, and get $k \geq \frac{2}{1-\gamma} \log(\frac{D_0}{\epsilon_1})$. Finally, we use a union bound argument to show that by setting a sufficient $t_0$, namely $t_0 = \Omega ( \frac{|S| |A| \cdot \max(M,K) (4R_{max} + 4 \gamma V_{max})^2 \cdot \log(\frac{|S||A| N}{\delta})}{c \cdot \epsilon^2 \epsilon_1^2 (M+K)} )$, w.h.p the noise term is bounded: $|W_m| < \frac{\epsilon}{3}D_k$,  $\forall k, m \in [t_{k+1}, t_{k+2}]$. Taking this $t_0$ and $t_k = 3 \cdot t_{k-1}$ gives the result.
\paragraph{Asynchronous setting}
We specify the exact changes required in the proof in appendix~\ref{sec:corollaryProof}. The main idea here, is that we have a longer 'cycle' length until all state-action pairs are updated, compared to a single iteration as in the synchronous case. We use again assumption~\ref{assum:ConditionsForConvergence} in the following way - we notice that the 'cycle' time until all pairs are updated is $\frac{1}{c}|S||A|$ iterations.
This allows us to use a modification of the previous proof and gives corollary~\ref{corollary:convergenceSingleUpdatePerStep}. 
\section{Experience replay in environments with rare experiences}
\label{sec:rareExperiences}
The results in the previous section provide a convergence guarantee for Q-learning with experience replay, which is generally comparable to standard Q-learning. However, they do not imply that experience replay actually improve performance. In this section, we illustrate how experience replay is \emph{provably} beneficial in a simple class of MDPs. In particular, we consider MDPs in which some transitions are rare, and show that at least for certain finite time horizons, Q-learning may not reach close to the optimal Q values, whereas sufficient experience replay following the online learning phase will result in convergence to the exact optimal Q values. This is due to the nature of the sequential updating of Q-learning and the correlation in the sequence of updates in terms of the states visited - outcomes of rare experiences will take a long time to propagate through the Q values table and this will lead to Q values that are far from optimal. This theoretical analysis illuminates a long-suggested benefit of experience replay - its ability to break the correlation structure of the Q-value updates, that emerges from the time and space correlation of samples drawn from the MDP. This analysis suggests that replay can be especially beneficial in propagating the outcomes of rare experiences throughout the Q-values, and raises the idea of choosing a replay sampling scheme that emphasizes this. A related empirical work by \citep{lee2018sample} showed that replaying entire episodes that end with significant outcomes in reverse order improves performance, in line with our theoretical findings here.
\begin{figure}[H]
\begin{center}
\centerline{\includegraphics[scale=0.25]{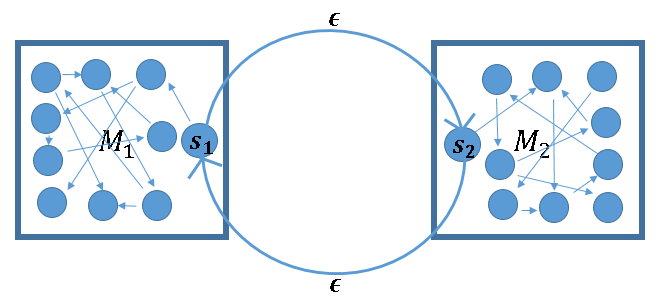}}
\caption{MDP class}
\label{fig:MDPrare}
\end{center}
\end{figure}
We start by defining the class of MDPs that will be discussed. This environment is inspired by works in the field of Neuroscience. Specifically, we look at environments similar to those used in \citet{momennejad2017offline}, where humans learned to perform an RL task while recordings of functional activations in the brain were taken via fMRI. \citet{momennejad2017offline} show that experience replay was crucial for learning a simple MDP where some transitions were witnessed very rarely. Our theoretical results provide some formal insight to these experimental observations and to the usefulness of experience replay, even in the tabular case. Figure~\ref{fig:MDPrare} illustrates the MDP we consider.
\paragraph{Class of MDPs} 
We consider the following setting - Let $M_3$ be an MDP with a state space $S$ and action set $A$. $M_3$ can be decomposed to 2 distinct MDPs $M_1, M_2$ with state spaces $S_1, S_2$ such that $S_1 \bigcap S_2 = \emptyset, S_1 \bigcup S_2 = S$, with the following property: there exists a single state (which we term $s_1$) in $S_1$ and a single state (termed $s_2$) in $S_2$ for which the transition probability from one to the other, i.e. $P(s_1|s_2), P(s_2|s_1)$, is non-zero. This means that $M_3$ actually consists of 2 separate environments that are sparsely connected to one another. For simplicity we further assume that there exists one possible action in states $s_1,s_2$ (we do not assume this on other states), although this assumption can be easily lifted by a slight modification to the proof.
In order to consider environments with rare experiences, we consider transition matrices in which with probability $\epsilon$ we transition between $s_1$ and $s_2$ and vice versa, where $\epsilon$ ensures that this transition is in fact rare with respect to the time horizon, and is specified in Theorem~\ref{theorem:rareExperiences}. Otherwise, with probability $1-\epsilon$, the following state is some specific state in the same part of the MDP (i.e., $M_1$ for $s_1$ and $M_2$ for $s_2$). Finally, we consider 'easy' MDPs that are almost deterministic - the reward $r(s,a)$ is a deterministic function and the transition matrix for all states except $s_1,s_2$ (which are stochastic states of transition from one part of the environment to another) is deterministic, i.e. for states $s \in S \setminus \{s_1,s_2\}$, we have a deterministic transition matrix.

We consider a setting in which we run Q-learning for some fixed number $T'$ of iterations (specified in Theorem~\ref{theorem:rareExperiences}) starting from some state in $M_1$, and study the resulting Q values, over the randomness of the algorithm and the environment.
We begin with some assumptions.
\begin{assumption}
\label{assum:assumptionsRareExperiences}
We impose the following assumptions on $M_3$:
\begin{itemize}[leftmargin=*]
\item There exists a time horizon $T$ such that if Q-learning was run on $M_1$ or on $M_2$ separately (i.e. when $\epsilon = 0$) for $T$ iterations, then w.h.p over the randomness of the algorithm all state-action pairs (in $M_1$ or $M_2$) are visited at least once.
\item Denote by $Q_i^\ast$ the optimal Q values for $M_i$ alone, for $i=1,2$, when the two environments are unconnected, i.e. when $\epsilon = 0$. The optimal Q values of $M_i$ and $M_3$ are far from one another: $\forall s,a : |Q_i^\ast(s,a) - Q_3^\ast(s,a)| \geq D_0$, where $D_0$ is some constant.
\item The convergence rate of running Q-learning on $M_3$ is captured by a function $g$ such that after $t$ iterations, $\|Q_{3,t} - Q^\ast_3 \|_\infty \leq g(t)$
\end{itemize} 
\end{assumption}
The assumption above is similar to the assumption used in previous sections as well as in previous works \citep{even2003learning}. First, we assume a time horizon $T$ that ensures we see all state-action pairs. In an ergodic MDP, for any Q-learning algorithm to converge such a $T$ must exist. 
The second assumption makes sure that in order to understand the full environment of $M_3$, one needs to combine knowledge from $M_1$ and $M_2$. This means that we want $M_{1}, M_2$ to be different enough from $M_3$, giving rise to a large distance between the optimal Q values in $M_1, M_2$ and the optimal values in $M_3$.
Finally, the last assumption is simply a notation for the convergence rate of running Q-learning on $M_3$.
\begin{figure*}
\centering
\begin{subfigure}
\centering
\includegraphics[scale=0.41]{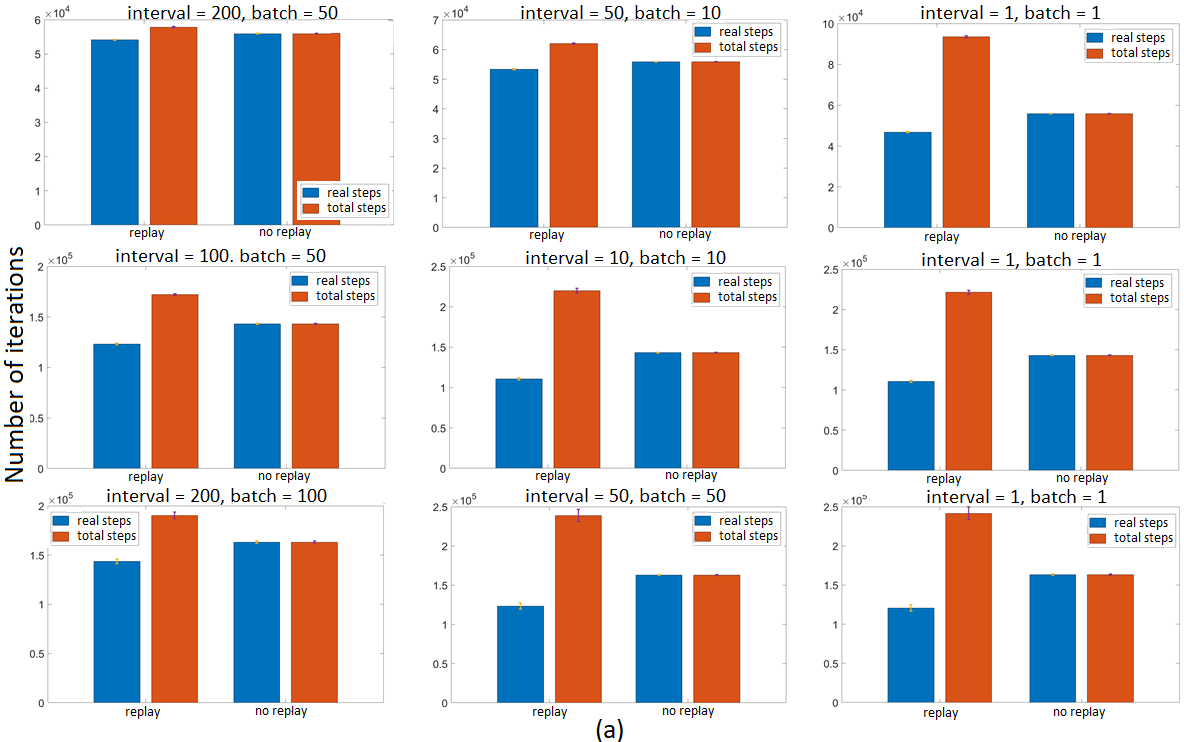}
\label{fig:exp1}
\end{subfigure} 
\hfill
\begin{subfigure}
\centering
\includegraphics[scale=0.25]{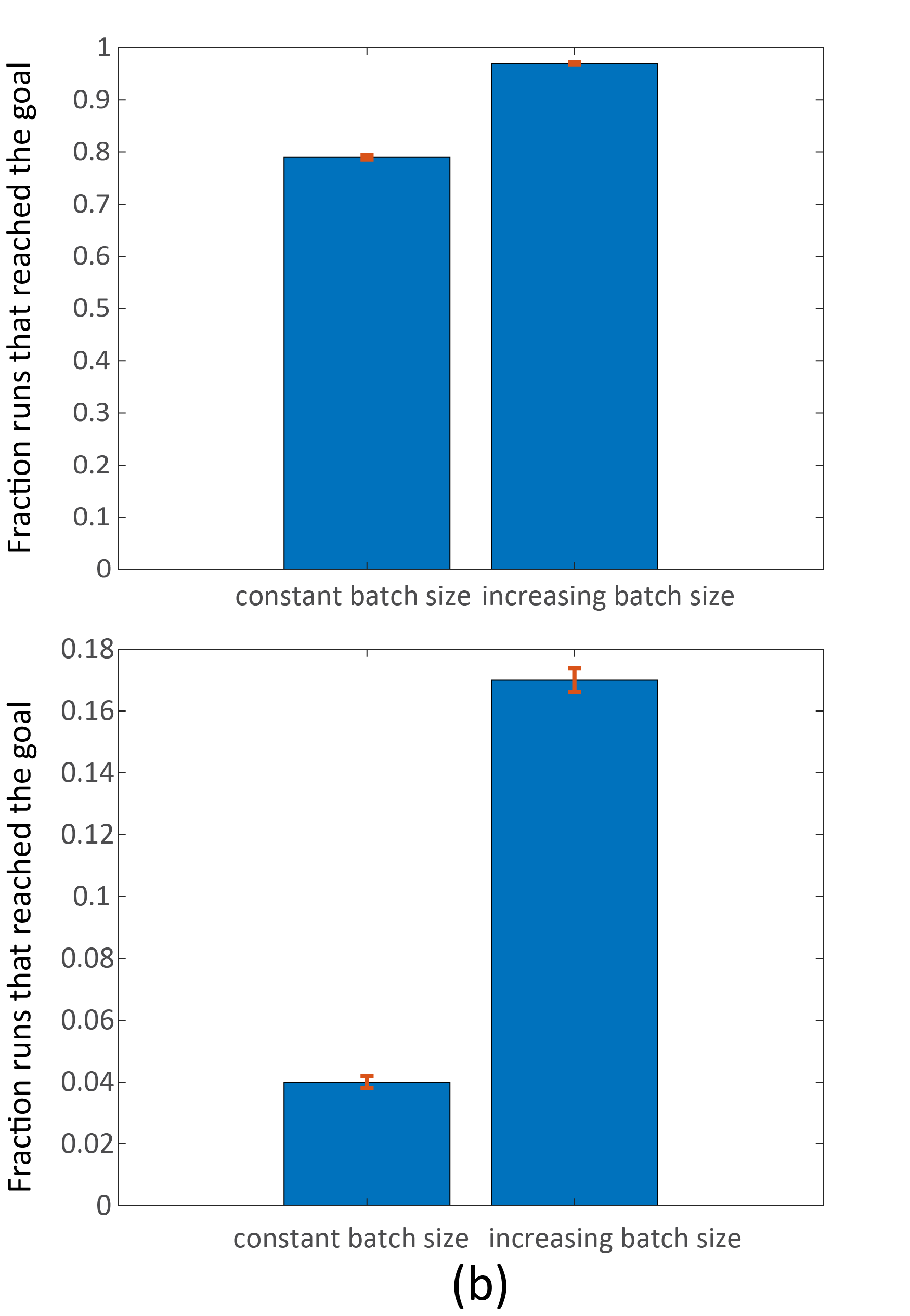}
\label{fig:exp2}
\end{subfigure}
\caption{(a) Convergence of Q-learning with and without replay, averaged over 100 repetitions. First and second row- convergence to score threshold, third row- convergence of Q values. Top row- Medium grid world. Second and third row- Hard grid world. (b) Fraction of runs that reached the goal in the Medium (top) and hard (bottom) grid world, over 100 repetitions. Error bars are SEM.}
\label{fig:scoreResults}
\end{figure*}
\begin{theorem}\label{theorem:rareExperiences}
Let $M_3$ be an MDP from the class described above, under assumption~\ref{assum:assumptionsRareExperiences}. 
Define $T' = \max(20000, 100 T)$ to be the time horizon. Let $\psi \in (0,1)$. Let $p > 0.00009$ denote the expected frequency of visiting $s_1$ and $s_2$, where expectation is taken over multiple runs of length $T'$ of the algorithm and the randomness of the algorithm and the environment, and suppose that $\epsilon = \frac{1.73}{T' \cdot p}$. 
Then, with probability $\geq \frac{1}{2}$, 
$\exists s,a : | Q_{3,T'}
(s,a) - Q_3^\ast(s,a) | \geq \frac{D_0}{2}$
Furthermore, with probability $\geq \frac{1}{2}$, running Q-learning for $T'$ iterations followed by $g^{-1}(\psi)$ iterations of experience replay in a frequency similar to that of the MDP $M_3$ converges to $\psi$ proximity of $Q^\ast_3$.
\end{theorem}
The proof follows the following intuitive idea: any Q-learning algorithm, which only uses online experience, will have to transition between the two parts of the environment sufficiently many times in order to have all state-action pairs account for the expected returns of the two parts of the environment.
However, since we are considering a finite time horizon and transitions that are rare with respect to that time horizon, it is likely that only few transitions between the two parts of the environment occur during this finite time. We show that with probability $\geq \frac{1}{2}$, we have at most one transition between $s_1$ and $s_2$, and one transition back. We then show that if we have at most two transitions - then some Q-values converge to the optimal Q values for one part of the environment alone, $Q^\ast_1$ or $Q^\ast_2$, and not to $Q^\ast_3$. This is because the sequence of experiences we observe is indistinguishable from a sequence that we would see if the two parts of the environment were unconnected (i.e. $\epsilon = 0$). However, we show that with probability at least $\frac{1}{2}$ we have spent at least $T$ iterations in each part of the environment, and thus have seen all states and actions in $M_1, M_2$. Thus, when applying experience replay after the online learning phase we will be able to converge to the exact optimal Q values, since we have a complete representation of the environment in the memory.
The full proof is found in appendix~\ref{sec:rareExperienceProof}.
\section{Experiments}
\label{sec:experiments}
In the previous section, we showed that for a particular simple class of MDPs, experience replay is provably beneficial. We now complement this by showing empirically how experience replay can be beneficial, on more complicated MDPs, while supporting the results of theorem~\ref{theorem:convergenceRate} empirically.
Here, we show that as discussed before, when we replay extensively, the total number of iterations needed for convergence increases, and this effect is monotonically demonstrated with lower $K$ (interval between replay events) or higher $M$ (number of replay iterations). However, a possible benefit of experience replay is also demonstrated - while the number of total iterations for convergence might increase due to replay, the number of real-world samples might decrease. This means that replay iterations might save us real-world iterations, that could be costly or hard to achieve.\\
In a second experiment, we demonstrate how our theoretical results translate into practical suggestions for improving the replay usage. In this experiment, we show that as suggested by our theoretical results, the error accumulation due to replay stems from the difference in distributions in the buffer and the MDP. Driven by this finding, we compare Q-learning with replay under two different schedules - one in which the replay batches are uniform throughout the duration of learning, and one in which they are increasing as time passes and more experiences are collected to the memory buffer. We show that as expected, replaying in an increasing schedule has superior performance. \\ 
We use the medium and hard grid world environments from \cite{zhang2017deeper}. 
Each grid is a 20-by-20 grid with walls in some of the cells, in which the task is to reach from a starting position to a goal position (see figure~\ref{fig:experiemntGrids} in appendix~\ref{app:gridWorld}). There are 4 possible actions- up, down, left, right. Each step incurs a reward of $-1$, and the goal position incurs a reward of $0$. When reaching the goal a transition back to the starting position takes place. We call each trajectory that starts in the starting position and ends in the goal position an episode.
We run Q-learning with and without experience replay in both environments, with various replay intervals and number of replay iterations ($K$ and $M$). We test for convergence in two ways- first, we look at the score of every trajectory from the starting position to the goal position, and look for when does this score passes a threshold (-50 for the medium grid, -70 for the hard grid) that indicates that a trajectory close to optimal has been learned. Second, we look at the Q values themselves and check when the Q values have converged to a threshold of $0.0001$, over all state-action pairs. We look at the number of online steps vs the total number of steps taken until convergence, averaged over $100$ repetitions. 
Figure~\ref{fig:scoreResults}a displays the results. The first and second row show convergence to the score threshold. The third row shows the convergence of the Q values. Our results clearly indicate that at least in these environments, employing experience replay saves real-world iterations, and increases the number of total iterations until convergence, and this effect is monotonically increasing with the number of replay iterations.\\
In a second experiment, we compare two schedules of replay- one in which $50$ replay iterations take place every $100$ iterations, and another in which we use an increasing batch size in intervals of $100$ iterations, starting from very small batches that increase in time. In both schedules we learn for $750000$ total iterations, of which $250000$ are replay iterations. We then test the performance of the resulting policies. Figure~\ref{fig:scoreResults}b shows that an increasing schedule outperforms a constant batch size, as expected. This is as implied by our theoretical analysis, and is an important practical finding.
\section{Discussion}
\label{sec:discussion}
Experience replay is an important tool in the toolkit of RL. It is also a biological phenomenon, which takes places naturally in human and animal learning. Thus, it is valuable to understand the advantages and disadvantages of this heuristic, and theoretically investigate its properties. To the best of our knowledge, this is one of the first works probing these interesting questions. We start by proving, for the first time, the convergence rate of Q-learning with experience replay, a widely used algorithm. Perhaps surprisingly, we find that Q-learning with experience replay generally has a comparable learning rate to that of Q-learning, but in extensive usage of experience replay the convergence rate might increase. Our analysis also points to the origin of this increase- sampling from the memory buffer incurs the additional error of sampling from an empirical distribution instead of the actual one. These results are also demonstrated empirically, suggesting that this phenomenon is also relevant in practice. These results give intuition as to how experience replay should be used - when sampling from the replay buffer, one should take into account the difference between the empirical distribution and the actual one. This raises the idea of increasing the frequency of replay as more experiences are obtained and a more accurate approximation of the true transition distribution is available, as we empirically demonstrate to be beneficial. A question that arises is the way in which controlling the size of the memory buffer might effect the distance between the transition distribution in the memory buffer compared to the actual one, and consequently the convergence rate. \citet{liu2017effects} showed empirically that a memory buffer too small deteriorates performance, and our results suggest a possible explanation for this empirical observation. 
In a complementary analysis, we further investigate the benefits of experience replay. We show that replay is provably beneficial in environments with rare experiences, where replay allows to quickly propagate information between distant and scarcely connected parts of the environment, by breaking the time and space correlation structure of the Q-value updates. This brings up an interesting question of the interplay between experience replay and multi-step methods, that somewhat mitigate the effect of learning from local information as done in single-step methods. Our results suggest that it is possible that experience replay can, to some extent, substitute the role of multi-step updates in that sense. Another interesting idea that arises following our results is that of applying replay in reverse order of a trajectory leading to a rare or highly unexpected experience, and by doing so swiftly propagating these rare experience backwards to the rest of the Q values. 
In summary, we believe that this paper is an important step in the theoretical investigation of experience replay, a highly useful tool in reinforcement learning. 
\newpage
\bibliography{biblio_memory_replay}
\bibliographystyle{plainnat}

\newpage
\onecolumn
\section{Appendix}
\label{sec:proofs}
\subsection{Q-learning with experience replay algorithm}
\label{app:alg}
The algorithm for Q-learning with experience replay:
\begin{algorithm}[h]
\caption{Q-learning with uniform experience replay}
\label{alg:Q-learning with replay}
\begin{algorithmic}
\STATE input: $M=$ number of replay iterations, $K=$ interval length between replay events, $\gamma\in (0,1) =$ decay parameter, $\alpha_t(\cdot, \cdot)$ = learning rate function
\STATE init: $\forall a \in A, s \in S: Q_0(s,a) = c_0, t=0$
\FOR{$i = 1,2,...$}{
\STATE $t = t+1$
\STATE Choose action $a_t = argmax_a (Q_{t-1}(s_t,a))$ (or explore)
\STATE Receive $r_t$, transition to $s_{t+1}$, and update:\\
$Q_t(s_t,a_t) = (1-\alpha_t(s_t,a_t)) \cdot Q_{t-1}(s_t,a_t) + \alpha_t(s_t,a_t) \cdot \left( r_t + \gamma \cdot max_{a'} Q_{t-1}(s_{t+1},a') \right)$
\STATE Store transition in memory: $s_t, a_t, r_t, s_{t+1}$
\IF {$i_{\mbox{mod K}} == 0$}  {
\FOR {$j=1,...,M$}{
\STATE $t = t+1$
\STATE Sample transition from memory uniformly at random: $s_1, a, r, s_2$
\STATE Update: $Q_t(s_1,a) = (1-\alpha_t(s_1,a)) \cdot Q_{t-1}(s_1,a) + \alpha_t(s_1,a) \cdot \left( r + \gamma \cdot max_{a'} Q_{t-1}(s_2,a') \right)$
}
\ENDFOR
}
\ENDIF 
}
\ENDFOR
\end{algorithmic}
\end{algorithm}

\subsection{Synchronous Q-learning with experience replay}
\label{app:synchronousQlearning}
The synchronous setting of Q-learning with experience replay is described in the following algorithm:
\begin{algorithm}[H]
\caption{Synchronous Q-learning with uniform experience replay}
\label{alg:synchronous Q-learning with replay}
\begin{algorithmic}
\STATE input: $M=$ number of replay iterations, $K=$ interval length between replay events, $\gamma\in (0,1) =$ decay parameter.
\STATE init: $\forall a \in A, s \in S: Q_0(s,a) = c_0, t=0$.
\FOR{$i = 1,2,...$}{
\STATE $t = t+1$
\FOR{all $s \in S, a\in A$}{
\STATE Sample from the MDP: $r_t(s,a)$, $s_{t+1}(s,a)$
\STATE Let $\alpha_t(s,a) = \frac{1}{n(s,a)}$ where $n(s,a)$ = number of updates done on $Q(s,a)$ so far 
\STATE Update:\\
$Q_t(s,a) = (1-\alpha_t(s,a)) \cdot Q_{t-1}(s,a) + \alpha_t(s,a) \cdot \left( r_t(s,a) + \gamma \cdot max_{a'} Q_{t-1}(s_{t+1},a') \right)$
\STATE Store transition in memory: $s, a, r_t, s_{t+1}$
}
\ENDFOR
\IF {$_{\mbox{mod K}} == 0$}  {
\FOR {$j=1,...,M$}{
\STATE $t = t+1$
\FOR{all $s \in S, a\in A$}{
\STATE Sample uniformly at random a transition of the pair (s,a) from all samples of it in the memory: $s, a, r, s_2$
\STATE Let $\alpha_j = \frac{1}{n(s,a)}$ where $n(s,a)$ = number of updates done on $Q(s,a)$ so far
\STATE Update: $Q_t(s_1,a) = (1-\alpha_j) \cdot Q_{t-1}(s_1,a) + \alpha_j \cdot \left( r + \gamma \cdot max_{a'} Q_{t-1}(s_2,a') \right)$
}
\ENDFOR
}
\ENDFOR
}
\ENDIF 
}
\ENDFOR
\end{algorithmic}
\end{algorithm}

\subsection{Grid world scheme}
\label{app:gridWorld}
The medium and hard grid worlds are described in fig~\ref{fig:experiemntGrids}. The yellow cells represent walls that cannot be crosses, the blue sell indicates the starting position, and the orange cell indicates the goal position.
\begin{figure*}[h!]
\centering
\includegraphics[width=12cm]{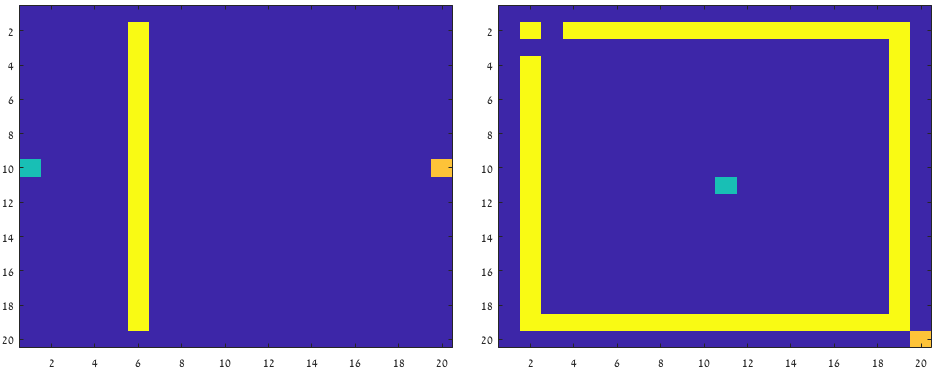}
\caption{Grid world environments of experiments \citep{zhang2017deeper}. Left: medium grid world, right: hard grid world. Starting position is the light blue cell, goal position is the orange cell, walls are marked in yellow.}
\label{fig:experiemntGrids}
\end{figure*}

\subsection{Proof of Theorem~\ref{theorem:convergenceRate}}
\label{app:convergenceRateProof}
We discuss the proof of Theorem~\ref{theorem:convergenceRate}. Since we are in the synchronous setting, we assume from this point on that every state-action pair is being updated at every iteration, and thus $\alpha_t(s,a) = \frac{1}{t}$. The explicit algorithm for the synchronous setting is specified in appendix~\ref{app:synchronousQlearning}.

We use the framework of iterative stochastic algorithms. For completeness, we start by showing how Q-learning with experience replay can be written as an iterative stochastic algorithm. This is similar to the formulation in \cite{bertsekas1996neuro} and \cite{even2003learning}.
An iterative stochastic algorithm is of the form: \\
$X_{t+1}(i) = (1-\alpha_t(i)) \cdot X_t(i) + \alpha_t(i) \cdot \left( H_tX_t(i) + w_t(i) \right)$ \\
where $H_t$ is a contraction (or pseudo-contraction) mapping, selected from a family of mappings $\mathcal{H}$ as a function of information contained in the past history $\mathcal{F}_t$.
A function $H: \mathbb{R}^n \rightarrow \mathbb{R}^n$ is a maximum norm contraction if there exists some constant $\beta \in [0,1)$ such that
$\forall x,\bar{x}: \|Hx - H\bar{x} \|_\infty \leq \beta \|x-\bar{x} \|_\infty$

Q-learning with experience replay can be written as an iterative stochastic algorithm (proof in Appendix~\ref{app:IterativeStocahsticAlg}) using the contraction mapping $H$ over Q functions, defined as: 
\begin{align*}
HQ(s,a) = \sum_{s' \in S} p_{s,s'}(a) \left( R(s,a) + \gamma \cdot max_{a'} Q(s',a') \right)
\end{align*} 
and $w_t$ defined as: 
\begin{align*}
w_t(s_t,a_t) = r(s_t,a_t) + \gamma \cdot max_{a'}Q_t(s_{t+1},a') - HQ(s_t,a_t)
\end{align*}
Combining the two we get:\\
$Q_{t+1}(s_t,a_t) = (1-\alpha_t(s_t,a_t)) \cdot Q_t(s_t,a_t) + \alpha_t(s_t,a_t) \cdot \left( HQ_t(s_t,a_t) + w_t(s_t,a_t) \right)$ \\
which is exactly the update rule in Q-learning.

The proof idea is as follows - we start by showing that the Q values can be decomposed to 2 parts, one which describes the deterministic process of convergence that would take place if all updates were done using the expected value of all random variables, and a second part that holds the error accumulated by the stochastic sampling of the random variables in question. Using this decomposition, we show that we can bound the distance between the Q values and the optimal Q values by smaller and smaller constants, which leads to convergence. This is a common technique used to prove convergence in iterative stochastic algorithms. However, since we are sampling both from the MDP and from a memory buffer, in our case the error accumulated by the sampling of memories does not have an expected value of $0$, as opposed to the error accumulated by online sampling from the MDP. This is the main challenge in proving the convergence of Q-learning with experience replay, and is the heart of our proof. 

We denote the history (all transitions, rewards and actions before iteration $t$) as $F_t$, i.e. $F_t = \{ s_1,a_1,r_1,s_2,a_2,r_2,...,s_{t-1},a_{t-1},r_{t-1} \}$. 
We first study the effect of the noise term $w_t$ on the convergence of the algorithm. For iterations of online Q-learning (where the samples are drawn from the MDP) the expected value of the noise term, $w$, is $0$.
\begin{align*}
\mathbb{E}\left[ w_t(s,a) | F_t \right] &= \mathbb{E} [ r(s,a) + \gamma \cdot max_{a'} Q_t(s_{t+1},a') \\
&- \sum_{s' \in S} p_{s,s'}(a) \cdot \left( R(s,a) + \gamma \cdot max_{a'} Q_t(s',a') \right) | F_t ] \\
&= \mathbb{E} \left[ r(s,a) \right] - R(s,a) \\
&+ \gamma \cdot \sum_{s' \in S} p_{s,s'}(a) \cdot max_{a'}Q_t(s',a') - \gamma \cdot \sum_{s' \in S} p_{s,s'}(a) \cdot max_{a'} Q_t(s',a') \\
&= 0
\end{align*}
On the other hand, for iterations where we sample from the memory buffer, we do not have that $\mathbb{E} \left[ w_t(s,a) | F_t \right] = 0$, since in these iterations the expectation of $r(s,a)$ and $s_{t+1}$ are not equal to the true expectations as in the MDP, but to the empirical means of the experiences collected in the memory buffer. 
Thus, in iterations where we sample from the memory, we have that:
\begin{align*}
\mathbb{E} \left[ w_t(s,a) | F_t \right] &= \mathbb{E} \left[ r(s,a) + \gamma \cdot max_{a'} Q_t(s_{t+1},a') | F_t \right] \\
&- R(s,a) - \gamma \cdot \sum_{s'=0}^n p_{s,s'}(a) \cdot max_{a'} Q_t(s',a') \\
&=\widehat{r(s,a)} - R(s,a) + \gamma \cdot \widehat{max_{a'}Q_t(s_{t+1},a')} - \gamma \cdot \sum_{s'=0}^n p_{s,s'}(a) \cdot max_{a'} Q_t(s',a')
\end{align*}
where $\widehat{r(s,a)}$ and $\widehat{max_{a'}Q_t(s_{t+1},a')}$ are the empirical means of the reward in state-action pair $(s,a)$ and the empirical maximal $Q$ value for the next state $s_{t+1}$ (averaged over the experiences in the memory buffer).

Explicitly, let $L$ be the set of all transitions $(s_1,a,s_2,r)$ in the memory buffer with state $s_1=s$ and action $a$. Then we can write these empirical means in the following manner:
\begin{align*}
&\widehat{r(s,a)} = \frac{1}{|L|} \sum_{r \in L} r \\
&\widehat{max_{a'}Q_t(s_{t+1},a')} = \frac{1}{|L|} \cdot \sum_{s_2 \in L} max_{a'}Q_t(s_2,a')
\end{align*}
In order to prove the above convergence rate, we first show the following two claims (proofs are found in appendix~\ref{app:boundedSeqProof} and~\ref{app:proofLemmaNoiseTermAsymptotic}).

\begin{restatable}{lemma}{boundedSeq}
\label{lemma:boundedSeq}
There exists a constant $D_0 \leq V_{max} + max(\| Q_0 \|_\infty,V_{max})$ such that $\forall t: \| Q_{t} - Q^\ast \|_\infty \leq D_0$  
\end{restatable}

\begin{restatable}{lemma}{noiseTerm}
\label{lemma:noiseTerm}
Let $W_0(s,a) = 0$ and $W_{t+1}(s,a) = (1-\alpha_t(s,a)) \cdot W_t(s,a) + \alpha_t(s,a) \cdot w_t(s,a)$. Then,
$lim_{t \rightarrow  \infty} W_t(s,a) = 0$ almost surely.
\end{restatable}
Lemma~\ref{lemma:boundedSeq} says that for all iterations, the distance between our Q values and the optimal ones, is bounded by some constant $D_0$.\\
Lemma~\ref{lemma:noiseTerm} ensures that even though we do not always sample from the MDP, we still have that the accumulated noise, $W_{t+1}(s,a)$, converges to $0$ as $t \rightarrow \infty$. 

We now write $Q(s,a)$ as a sum of the two components discussed above: $Y_t(s,a)$, the 'deterministic' error term, representing the convergence had we used expectations of random variables rather than stochastic samples, and $W_t(s,a)$, the 'stochastic' error, holding the error due to sampling from either the MDP or the memory buffer.
Using this decomposition, we recursively bound the difference $Q_t(s,a)-Q^\ast(s,a)$ by a sequence of distance constants defined as follows: $D_{k+1} = (\gamma + \epsilon)\cdot D_k$, for $\epsilon = \frac{1 - \gamma}{2}$. Clearly, $D_k \rightarrow 0$ as $k \rightarrow \infty$, thus this results in convergence to the optimal Q values. We prove that for each $k$, there exists a time $t_k$ such that for all iterations following $t_k$ our Q values are not further than $D_k$ from the optimal ones.

\convergenceInLimit*

The full proof of lemma~\ref{lemma:convergenceInLimit} is found in appendix~\ref{app:convergenceInLimitProof}. The proof follows by induction. 
First we notice that since $\gamma < 1$, $lim_{k \rightarrow \infty} D_k = 0$. 
We assume that there exists some $t_k$ such that $\forall t \geq t_k : \| Q_{t}(s,a) - Q^\ast(s,a) \| \leq D_k$, and show that there exists some $t_{k+1}$ for which
$\forall t \geq t_{k+1} : \| Q_{t}(s,a) - Q^\ast(s,a) \| \leq D_{k+1}$. \\
For $t_0=0$, we have that $\forall t > t_0 : \| Q_{t}(s,a) - Q^\ast(s,a) \| \leq D_0$, from Lemma~\ref{lemma:boundedSeq}, and thus the induction basis holds. \\
For every $s,a, t \geq t_k$, we show that the Q values are bounded:
\begin{equation}
\label{eq:decompositionOfQval2}
\begin{split}
-Y_t(s,a) - W_{t:t_k}(s,a) \leq Q_{t}(s,a) - Q^\ast(s,a) \leq Y_t(s,a) + W_{t:t_k}(s,a)
\end{split}
\end{equation}
Where we define $\forall \tau,s,a: W_{\tau:\tau}(s,a) = 0$ and $W_{t+1:\tau}(s,a) = (1-\alpha_t(s,a))W_{t:\tau}(s,a) + \alpha_t(s,a)w_t(s,a)$.
Additionally, let $Y_{t_k}(s,a) = D_k$ and $\forall t \geq t_k : Y_{t+1}(s,a) = (1-\alpha_t(s,a))Y_t(s,a) + \alpha_t(s,a) \gamma D_k$.
We then show that $lim_{t \rightarrow \infty} Y_t(s,a) = \gamma \cdot D_k < D_k$, and combine this with the results of lemma~\ref{lemma:noiseTerm} that states that $lim_{t \rightarrow \infty} W_t(s,a) = 0$. This gives us the required result.

We got that there exists some $t_k$ such that $\forall t \geq t_k : \| Q_{t}(s,a) - Q^\ast(s,a) \| \leq D_k$, and some $t_{k+1}$ for which
$\forall t \geq t_{k+1} : \| Q_{t}(s,a) - Q^\ast(s,a) \| \leq D_{k+1}$
Our goal now is to bound $t_{k+1}-t_k$.
We use equation~\ref{eq:decompositionOfQval2}.
We start by bounding the number of iteration for $Y_t(s,a)$ to become smaller than $(\gamma + \frac{2 \epsilon}{3}) \cdot D_k$, then we bound the number of iterations until $W_t(s,a)$ is smaller than $\frac{\epsilon}{3} D_k$, and this will give us the number of iteration to have $Y_t(s,a) + W_t(s,a) < D_{k+1}$.

\deterministicPartConvergenceTime*
The above lemma actually tells us what is the number of updates that needs to be done on a Q-value of a state-action pair if this update were done using the actual expected reward of each action in that state, and with the expected next state Q value, given the transition probabilities. This is in fact the number of iterations we would have needed if we were updating the Q-values with expectations of the real MDP, rather than with samples from it or from the memory buffer. We summarize that for all $t > 3 t_k$, $Y_t(s,a) < (\gamma + \frac{2 \epsilon}{3})D_k$, thus $t_{k+1}=3\cdot t_k$. The proof is found in appendix~\ref{app:proofYconvergenceTime}.

We now continue to bounding the time until $W_t \leq \frac{\epsilon}{3} D_k$. This is done by looking at the noise accumulated from $t_k$ until $t$, for all $t \in [t_k, t_{k+1}]$. We show that for all iterations in this set interval, the noise term is bounded by $\frac{\epsilon}{3} D_k$. We then show that this is true recursively, for all $k$, and combining this with lemma~\ref{lemma:deterministicPartConvergenceTime} gives us the required convergence rate.
In order to show our noise term is bounded, we separate online experiences and replay experiences, and bound the contribution of noise of each of these samples separately. We then show that by setting a compatible $t_0$, we get the required bound for all time points.

We start by reminding ourselves the form of $W_t$: 
\begin{align*}
W_t(s,a) &= (1-\alpha_{t-1}(s,a))W_{t-1}(s,a) + \alpha_{t-1}(s,a)w_t(s,a)
\end{align*}
and 
\begin{align*}
w_t(s_t,a_t) = r(s_t,a_t) + \gamma \cdot max_{a'}Q_t(s_{t+1},a') - \sum_{s'=0}^{n} p_{s,s'}(a_t) \cdot \left( R(s_t,a_t) + \gamma \cdot max_{a'} Q(s',a') \right)
\end{align*}
We would like to bound now the probability that for a time point $m \in [t_{k+1}, t_{k+2}]$, $W_m(s,a)$ is large, i.e.:
$\Pr(|W_m(s,a)| \geq \frac{\epsilon}{3}D_k)$.

First we note that $W_m$ is in fact a scaled average of samples $w_i$ for $i\in \{1,2,\ldots,m\}$, where some are online samples taken from the MDP, and some are samples taken from the memory buffer. 
We will denote samples taken from the MDP (i.e., real-world experience) as $q_i$, and samples drawn from the memory buffer as $g_i$. 
Since we are taking $M$ replay iterations every $K$ real world iterations, we have that out of the $m$ samples in $W_m$, $\frac{mM}{M+K}$ of then are replay samples, and $\frac{mK}{M+K}$ of them are samples from the MDP. 
We will omit in the following analysis the state and action pair, and we will refer to a single and constant pair, $(s,a)$.  
We can write $W_m$ as two sums, each summing over samples of one kind:
\begin{align*}
\Pr\left(|W_m| \geq \frac{\epsilon}{3}D_k\right) &= \Pr\left(\left| \frac{1}{t_k+m}\left(\sum_{i=1}^{\frac{mM}{M+K}}g_i + \sum_{i=1}^{\frac{mK}{M+K}}q_i \right) \right| \geq \frac{\epsilon}{3} D_k \right) \\
& \leq \Pr \left( \left| \frac{1}{t_k+m} \sum_{i=1}^{\frac{mM}{M+K}}g_i \right| \geq \frac{\epsilon}{6}D_k  \mbox{  or  } \left| \frac{1}{t_k+m} \sum_{i=1}^{\frac{mK}{M+K}}q_i \right| \geq \frac{\epsilon}{6}D_k \right) \\
& \leq \Pr \left( \left| \frac{1}{t_k+m} \sum_{i=1}^{\frac{mM}{M+K}}g_i \right| \geq \frac{\epsilon}{6}D_k \right) + \Pr \left( \left| \frac{1}{t_k+m} \sum_{i=1}^{\frac{mK}{M+K}}q_i \right| \geq \frac{\epsilon}{6}D_k \right)
\end{align*}

We will bound each part separately, starting with examples drawn from the MDP, $q_i$'s.
First, we remind the reader that for a transition from state $s$ to a sampled state $s'$ after acting $a$ and receiving a reward $r_i(s,a)$, we have that:
$q_i(s,a) = r_i(s,a) - R(s,a) + \gamma \cdot \left( max_{a'} Q(s',a') - \sum_{s' \in S} p_{s',s}(a) max_{a'} Q(s',a') \right)$
We further notice that $E[q_i]=0$, and that $q_i$ are sampled i.i.d. We also note that since $|r(s,a)| \leq R_{max}$ and $Q(s,a) \leq V_{max}$, we have that $|q_i| \leq 2R_{max} + \gamma 2 V_{max}$.
Using these properties, we can use Hoeffding's inequality to get:
\begin{align*}
&\Pr \left( \left| \frac{1}{t_k+m} \sum_{i=1}^{\frac{mK}{M+K}}q_i \right| \geq \frac{\epsilon}{6}D_k  \right) \leq 2 \exp\left(\frac{-\epsilon^2 D_k^2 (m+t_k)^2 (M+K)}{18 m K (4R_{max}+4 \gamma V_{max})^2}\right)
\end{align*}

We now turn to bound the sum of replay samples:
\begin{align*}
\Pr \left( \left| \frac{1}{t_k+m} \sum_{i=1}^{\frac{mM}{M+K}}g_i \right| \geq \frac{\epsilon}{6}D_k \right) &\leq \Pr \left( \frac{1}{t_k+m} \sum_{i=1}^{\frac{mM}{M+K}} \left| g_i \right| \geq \frac{\epsilon}{6}D_k \right) \\
& = \Pr \left( \frac{1}{t_k+m} \sum_{i=1}^{\frac{mM}{M+K}} \left( \left| g_i \right| - \mathbb{E}[|g_i| | F_i] \right) + \frac{1}{t_k+m} \sum_{i=1}^{\frac{mM}{M+K}} \mathbb{E}[|g_i| | F_i] \geq \frac{\epsilon}{6}D_k \right) \\
& \leq \Pr \left( \frac{1}{t_k+m} \sum_{i=1}^{\frac{mM}{M+K}} \left( \left| g_i \right| - \mathbb{E}[|g_i| | F_i] \right) \geq \frac{\epsilon}{12} D_k \right) \\
&+ \Pr \left( \frac{1}{t_k+m} \sum_{i=1}^{\frac{mM}{M+K}} \mathbb{E}[|g_i| | F_i] \geq \frac{\epsilon}{12}D_k \right) 
\end{align*}
We treat each term above separately, starting with the first term: $Pr \left( \frac{1}{t_k+m} \sum_{i=1}^{\frac{mM}{M+K}} \left( \left| g_i \right| - \mathbb{E}[|g_i| | F_i] \right) \geq \frac{\epsilon}{12} D_k \right)$. We notice that $\mathbb{E} \left[ |g_i| - \mathbb{E}[|g_i| |F_i] \right] = 0$, and that $||g_i| - \mathbb{E}[|g_i| |F_i]| \leq 2 \cdot \left( 2R_{max} + 2\gamma V_{max} \right)$, since every $|g_i|$ is upper bounded by this, and so the difference between a sample $|g_i|$ and the expectation conditioned on the past is also bounded. 
We define a sequence of partial sums: $G_j = \sum_{i=1}^{j} \left( \left| g_i \right| - \mathbb{E}[|g_i| | F_i] \right)$, where $G_0 = 0$. Note that $G_1, G_2, G_3,...,G_{\frac{mM}{M+K}}$ is a martingale, since $\mathbb{E}[G_{j+1} | G_j,...,G_0] = G_j + \mathbb{E}[\left| g_i \right| - \mathbb{E}[|g_i| | F_i]] = G_j$. We also have that $|G_{j+1} - G_j| \leq 2 \cdot (2R_{max} + 2\gamma V_{max})$. Thus we can use Azuma's inequality to get:
\begin{align*} 
&\Pr \left( \frac{1}{t_k+m} \sum_{i=1}^{\frac{mM}{M+K}} \left( \left| g_i \right| - \mathbb{E}[|g_i| | F_i] \right) \geq \frac{\epsilon}{12} D_k \right) \leq \exp \left( \frac{-\epsilon^2 D_k^2 (m+t_k)^2 (M+K)}{288 m M (4Rmax + 4 \gamma Vmax)^2} \right)
\end{align*}

We now look at the last term we need to bound, using a union bound we get:
\begin{align*}
\Pr \left( \frac{1}{t_k+m} \sum_{i=1}^{\frac{mM}{M+K}} \mathbb{E}[|g_i| | F_i] \geq \frac{\epsilon}{12}D_k \right) &\leq \Pr \left( \exists i : \mathbb{E}[|g_i| |F_i] \geq \frac{\epsilon}{12} D_k \frac{(m+t_k)(M+K)}{m \cdot M}\right) \\
&\leq \sum_{i=1}^{\frac{mM}{M+K}} \Pr \left( \mathbb{E}[|g_i| | F_i] \geq \frac{\epsilon}{12} D_k \frac{(m+t_k)(M+K)}{m \cdot M} \right)
\end{align*}
We now bound each term in the sum. We notice that the number of samples that we saw of a state $s$ and action $a$ depends on past iterations and on the Q-function throughout the time span, and thus this number is a random variable, i.e in the expression $\widehat{r(s,a)} = \frac{1}{|L|} \sum_{i=1}^{|L|} r_i$, $|L|$, the number of samples of the state-action pair that we have in the memory is in itself a random variable, and similarly for $\widehat{max_{a'}Q_t(s_{t+1},a')}$ .
We use the following lemma, where under assumption~\ref{assum:ConditionsForConvergence}, we have that at time t, $|L| \geq \frac{c \cdot t}{|S| \cdot |A|}$ for some constant c, i.e, that the number of transitions involving $s_t,a_t$ are at least a constant fraction of the total number of iterations. We show that with exponentially high probability, the empirical mean in the memory buffer is close to the actual mean. 
\begin{lemma}\label{lemma:meanOfUnknownSampleSize}
Let $r_m = \frac{1}{m} \sum_{i=1}^m r_i$ where $r_i$ is sampled i.i.d from a distribution with expected value $R$ and $|r_i| \leq R_{max}$. Let $A$ be a random variable with support in the interval $[c t, t]$ for some constant $c \in (0,1)$. \\
Then, $Pr(|r_A - R| > \epsilon) \leq 2 \exp(- \frac{ct \cdot \epsilon^2}{4 R^2_{max}})$
\end{lemma}
The proof of lemma~\ref{lemma:meanOfUnknownSampleSize} is found in appendix~\ref{app:meanOfUnknownSampleSize}.
Using lemma~\ref{lemma:meanOfUnknownSampleSize}, we can bound the difference between the empirical mean and the true mean $\forall i \in [1, \frac{mM}{M+K}]$:
\begin{align*}
&\Pr\left( \mathbb{E}[|g_i| | F_i] \geq \frac{\epsilon}{12} D_k \cdot \frac{(m+t_k)(M+K)}{m \cdot M} \right) \\
& \leq \Pr\left( \left|\widehat{r(s,a)} - R(s,a) \right|  \geq \frac{\epsilon}{24} D_k \frac{(m+t_k)(M+K)}{m \cdot M} \right) \\
&+ \Pr \left( \left| \gamma \left( \widehat{max_{a'}Q_t(s_{t+1},a')} - \sum_{s'=0}^n p_{s,s'}(a) \cdot max_{a'} Q_t(s',a') \right) \right| \geq \frac{\epsilon}{24} D_k \cdot \frac{(m+t_k)(M+K)}{m \cdot M} \right) \\
&\leq 4 \exp\left( \frac{-c \cdot t_k \epsilon^2 D_k^2 (m+t_k)^2 (M+K)^2}{|S| |A| 2304 m^2 M^2 (4R_{max}^2 + 4 \gamma^2 V_{max}^2)} \right)
\end{align*}
We can now get the bound for the sum:
\begin{align*}
&\sum_{i=1}^{\frac{mM}{M+K}} \Pr \left( \mathbb{E}[|g_i| | F_i] \geq \frac{\epsilon}{12} D_k \cdot \frac{(m+t_k)(M+K)}{m \cdot M} \right) \\
&\leq \frac{m M}{M+K} \cdot \left( 4 \exp\left( \frac{-c \cdot t_k \epsilon^2 D_k^2 (m+t_k)^2 (M+K)^2}{|S| |A| 2304 m^2 M^2 (4R_{max}^2 + 4 \gamma^2 V_{max}^2)} \right) \right)
\end{align*}

We are now ready to bound the probability of $W_m$ being too big, for any $m \in [t_{k+1}, t_{k+2}]$. We take the bounds we received for all terms, and get:
\begin{align}
\label{eq:bound}
\begin{split}
&Pr(|W_m | \geq \frac{\epsilon}{3}D_k) \\
&\leq 2 \exp\left(\frac{-\epsilon^2 D_k^2 (m+t_k)^2 (M+K)}{18 m K (4R_{max}+4 \gamma V_{max})^2}\right) \\
&+ \frac{m M}{M+K} \cdot \left( 4 \exp\left( \frac{-c \cdot t_k \epsilon^2 D_k^2 (m+t_k)^2 (M+K)^2}{|S| |A| 2304 m^2 M^2 (4R_{max}^2 + 4 \gamma^2 V_{max}^2)} \right) \right) \\
& + \exp \left( \frac{-\epsilon^2 D_k^2 (m+t_k)^2 (M+K)}{288 m M (4Rmax + 4 \gamma Vmax)^2} \right)
\end{split}
\end{align}
Simplification of the equation~\ref{eq:bound} is strictly technical, and is found in appendix~\ref{app:simplificationOfBound}. 
Eventually, we get that:
\begin{align*}
&Pr(|W_m | \geq \frac{\epsilon}{3}D_k) \leq \exp\left( \frac{-0.0007 \cdot c \cdot t_k \epsilon^2 D_k^2 (M+K)}{|S| |A| \cdot \max(M,K) (4R_{max} + 4 \gamma V_{max})^2} + log\left(12 \cdot \frac{t_k M}{M+K} + 3 \right) \right)
\end{align*}

We have bounded the probability for a single $m \in [t_{k+1},t_{k+}]$ to be very large. Now, we would like to have that for all $m \in [t_{k+1}, t_{k+2}]: |W_m| \leq \frac{\epsilon}{3} D_k$ with high probability $1-\delta$, where $\delta \in (0,1)$. Here we take $t_{k+1}$ as the starting point in time where we bound $W_m$ since this is when $Y_t(s,a)$ becomes small enough. We then ensure that this happens for all $k$, which will give us the bound for all time points. We use the union bound to get:
\begin{align*}
Pr(\forall m \in [t_{k+1}, t_{k+2}] : |W_m| \leq \frac{\epsilon}{3}D_k ) &= 1 - Pr(\exists m \in [t_{k+1}, t_{k+2}] : |W_m| \leq \frac{\epsilon}{3}D_k ) \\
& \geq 1 - \sum_{i=t_k}^{t_{k+2}} Pr(|W_i| \geq \frac{\epsilon}{3} D_k)
\end{align*}
If we want this to happen with probability $1 - \delta$, for all pairs of states and actions, then we need: 
\begin{align*}
Pr(|W_m| \geq \frac{\epsilon}{3} D_k) \leq \frac{\delta}{|S| |A| \cdot 9 t_k}
\end{align*}

For this to take place we need to have:
\begin{align*}
&\exp\left( \frac{-0.0007 \cdot c \cdot t_k \epsilon^2 D_k^2 (M+K)}{|S| |A| \cdot \max(M,K) (4R_{max} + 4 \gamma V_{max})^2} + log\left(12 \frac{t_k M}{M+K} + 3 \right) \right) \leq \frac{\delta}{|S| |A| \cdot 9 \cdot t_k} \\
&
\end{align*}
We thus get that it is enough to have:
\begin{align}
\label{eq:t_k}
\begin{split}
&t_k = \Omega \left( \frac{ |S| |A| \max(M,K) (R_{max} + \gamma V_{max})^2  \log(\frac{|S||A|}{\delta})}{c \cdot \epsilon^2 D_k^2 (M+K)} \right)
\end{split}
\end{align}

The following two lemmas complete the proof - first, in order to converge to the desired accuracy, we need to compute for what $k$ do we have that $D_k \leq \epsilon_1$. 
\begin{lemma}\label{lemma:numberIterations}
Let $D_{k+1} = (\gamma + \epsilon) \cdot D_k$, where $D_0$ is some constant. Then, given $\epsilon_1$, for all $k \geq \frac{2}{1 - \gamma} log(\frac{D_0}{\epsilon_1})$, $D_k \leq \epsilon_1$
\end{lemma}
\begin{proof}
For all k:
$D_k = (\gamma+\epsilon)^k \cdot D_0$.
Taking the log over both sides, we get:
\begin{align*}
k \geq \frac{\log(\frac{D_0}{\epsilon_1})}{\log(\frac{1}{\gamma + \epsilon})} = \frac{\log(\frac{D_0}{\epsilon_1})}{\log(\frac{2}{1+\gamma})}
\end{align*}
where the last equality is due to the fact that $\epsilon = \frac{1-\gamma}{2}$. 
Note that $\log(\frac{1+\gamma}{2}) \leq \frac{1+\gamma}{2} -1 = \frac{\gamma-1}{2}$
and thus: $\log(\frac{2}{1+\gamma}) \geq \frac{1 - \gamma}{2}$. We get that:
$\frac{2}{1-\gamma} \geq \frac{1}{\log(\frac{2}{1+\gamma})}$, and thus we can upper bound the required $k$ as follows: 
$k \geq \frac{2}{1-\gamma} \cdot \log(\frac{D_0}{\epsilon_1})$
\end{proof}

To concludes the calculation of the time to converge, we need that our noise term is small enough during the entire time horizon, i.e. for all time periods $t_k$, from $k=0$ to $\frac{2}{1-\gamma} \cdot \log(\frac{D_0}{\epsilon_1})$. The following lemma states that given a sufficient $t_0$, we have the required upper bound on the noise term for all time periods. The proof appears in appendix~\ref{app:noiseTermConvergenceTimeProof}.

\begin{lemma}\label{lemma:NoiseTermConvergenceTime}
Let $W_{t_k:t} = (1-\frac{1}{t-1}) \cdot W_{t_k:t-1} + \frac{1}{t-1} \cdot w_{t-1}$, where $W_{t_k:t_k} = 0$. \\  
Given $\delta \in (0,1)$, $\epsilon_1 \in (0,1)$, $\gamma \in (0,1)$, and $D_{k+1} = (\gamma + \epsilon) \cdot D_k$, we have that for $\epsilon = \frac{1-\gamma}{2}$, for $N = \frac{2}{1-\gamma} \cdot \log(\frac{D_0}{\epsilon_1})$ where $D_0$ is a constant:
\begin{align*}
&Pr(\forall k \in [0,N], t \in [t_{k+1}, t_{k+2}] : |W_{t_k:t}| \leq \frac{\epsilon}{3}D_k) \geq 1-\delta
\end{align*} 
given that: \\
 $t_0 = \Omega \left( \frac{|S| |A| \cdot \max(M,K) (4R_{max} + 4 \gamma V_{max})^2 \cdot \log(\frac{|S||A| N}{\delta})}{c \cdot \epsilon^2 \epsilon_1^2 (M+K)} \right)$
\end{lemma} 

We finalize the proof of Theorem~\ref{theorem:convergenceRate} by the following argument:
Given $\delta \in (0,1)$, $\epsilon_1 \in (0,1)$, $\gamma \in (0,1)$, $ \epsilon = \frac{1-\gamma}{2}$. Let
\begin{align*} 
&t_0 = \Omega \left( |S| |A| \max(M,K) (R_{max} + \gamma V_{max})^2 \frac{ \log(\frac{|S||A|}{\delta})+\log(\frac{2}{1 - \gamma} log(\frac{D_0}{\epsilon_1})) }{c \cdot (1-\gamma)^2 \epsilon_1^2 (M+K)} \right)
\end{align*}
and $t_k = 3 \cdot t_{k-1}$, for all $k \in \left[0,  \frac{2}{1-\gamma} \cdot \log(\frac{D_0}{\epsilon_1}) \right]$ where $D_0$ is a constant. Then, for all $t \in [t_k,t_{k+1}]$: $\|Q - Q^\ast\|_\infty \leq D_k$ where $D_{k+1} = (\gamma + \epsilon) \cdot D_k$. 

Thus, is we take:\\
$T = \Omega\left( 3^{ \frac{2}{1 - \gamma}log(\frac{D_0}{\epsilon_1})} \cdot \frac{|S| |A| \cdot \max(M,K) (R_{max} + \gamma V_{max})^2 \cdot \left(\log(\frac{|S||A|}{\delta})+\log(\frac{2}{1 - \gamma} log(\frac{D_0}{\epsilon_1}))\right) }{c \cdot (1-\gamma)^2 \epsilon_1^2 (M+K)} \right)$, then we get that for $t \geq T: \|Q_t - Q^\ast\|_\infty \leq \epsilon_1$ w.p at least $1-\delta$.

\subsection{Proof of Theorem~\ref{theorem:rareExperiences}}
\label{sec:rareExperienceProof}

We start by introducing notations that will be used throughout the proof. We denote by $Q_1^\ast, Q_2^\ast$ the optimal Q values for $M_1, M_2$, respectively, when the two environments are unconnected, i.e. when $\epsilon = 0$. Similarly, we denote by $Q_3^\ast$ the optimal Q function for the general environment $M_3$. We use the superscript $t$ to denote the Q value at iteration $t$, i.e., $Q_i^t$ is the Q value at iteration t, when running Q-learning on the MDP that contains only environment $i$ (i.e., running on $M_1$ separately, $M_2$ separately, or $M_3$). 

We will use the following notation to indicate the convergence rate of the standard Q-learning - at time $t$, we have that $|Q_i^t(s,a) - Q_i^\ast(s,a)| \leq f(t) \cdot |Q_i^0(s,a) - Q_i^\ast(s,a)|$, i.e, after $t$ updates of the pair $s,a$, we converge in a rate of $f(t)$ to the optimal Q values, where $f$ is some monotonic decreasing function in $t$. We also assume that $f(\frac{T}{|S_2| |A_2|}) \leq \frac{1}{2}$, $f(\frac{T}{|S_1| |A_1|}) \leq \frac{1}{2}$. This assumption does not limit the range of functions $f(t)$ that describe the rate of convergence. 

We initialize the Q values in such a way that for all state-action pairs, we have that $|Q_3^0(s,a) - Q_2^\ast(s,a)| \leq D_0$, $|Q_3^0(s,a) - Q_1^\ast(s,a)| \leq D_0$, where $D_0$ is a constant. This constant is a lower bound on the distance between the optimal Q values in $M_2$ and in $M_3$, as mentioned in assumption~\ref{assum:assumptionsRareExperiences}. 

Let $\delta \in (0,1)$. We take $K=\frac{3}{\delta}$ and define a time horizon $T' = KT$, where $T$ is the time interval under which we see all state-action pairs in $M_1$ and $M_2$, as assumed in assumption~\ref{assum:assumptionsRareExperiences}.  
We take $\epsilon = \frac{2}{T' \cdot p}$, where $p$ is the expected frequency of visiting $s_1$ and $s_2$, as specified in Theorem~\ref{theorem:rareExperiences}.

Denote the number of transitions between $M_1$ and $M_2$ and vice-verse as $N$, and denote by $\tau_1, \tau_2$ the number of iterations spent in $M_1$, $M_2$, respectively.
We start by showing that if $N \geq 2$, then with probability $\geq \frac{1}{2}$, Q-learning with experience replay will converge to a $\psi$ proximity to the optimal Q values exactly.
Note that in every iteration $t$, the probability to transition from $M_1$ to $M_2$ and vice verse, is $\epsilon \cdot \Pr(s_t \in [s_1, s_2] )$. We also note, that we have that $\Pr(s_t \in [s_1, s_2] ) = p$, the expected frequency of vising $s_1, s_2$. Thus we have that the probability that there are at least 2 transitions between $M_1$ and $M_2$ and vice verse is: 
\begin{align*}
Pr(N \geq 2) &= 1 - \left( {T' \choose 0} \cdot (\epsilon \cdot p)^0 \cdot (1-\epsilon \cdot p)^{T'} + {T' \choose 1} \cdot (\epsilon\cdot p)^1 \cdot (1-\epsilon \cdot p)^{T' - 1} \right) \\
& = 1 - (1-\epsilon \cdot p)^{T'} - T' \cdot \epsilon \cdot p \cdot (1-\epsilon \cdot p)^{T' - 1}
\end{align*}
Taking $\epsilon = \frac{1.73}{T' \cdot p}$, we have that $Pr(N \geq 2) \geq \frac{1}{2}$ for any $T' \geq \max(20000, \frac{1.73}{p})$.

Thus, with probability $\geq \frac{1}{2}$ the number of transitions is at least 2.


We now show that given the event that $N=2$, for $K = \frac{3}{\delta}$, with probability at least $1-\delta$ we have at least $T$ iterations in $M_1$ as well as in $M_2$.

\begin{lemma}\label{lemma:enoughIterations}
Let $T' = KT$, and denote by $\tau_1, \tau_2$ the number of iterations spent in $M_1$, $M_2$, respectively.
Let $\delta \in (0,1)$, then, conditioned on $N=2$, for $K = \frac{3}{\delta}$, with probability at least $1-\delta$, $\tau_1 \geq T, \tau_2 \geq T$.
\end{lemma}
\begin{proof}
We denote $Pr(N=2) = q$. 
We now look at the $T' = k\cdot T$ iterations, and denote the first transition time as $I_1$, and the second transition time as $I_2$. 
Given that there are exactly 2 transitions, we have that all sequences of $T'$ strings in ${0,1}^{T'}$ in which there are exactly 2 indices with value $1$ and the rest are all $0$, have the same probability, i.e. we have that for every pair of distinct indices $x,y$, where $x < y$, we have: $Pr(I_1 = x,  I_2 = y | N=2) = \frac{1}{q} \cdot  p^2 \cdot (1-p)^{T'-2}$ where $p$ is the probability of transition in every iteration. 
We note that conditioned on having exactly 2 transitions, in order to have $\tau_1 \geq T$ and $\tau_2 \geq T$, we need to have $T \leq I_2 - I_1 \leq T' - 2T$. Since the pairs of indices are distributed uniformly, we just need to draw the interval $I_2 - I_1$ as received by the uniform distribution over distinct pairs of indices $I_1, I_2$, and calculate the probability of this interval to be in the required range, $ I_2 - I_1 \in [T, T' - 2T]$.

In order to uniformly draw these two indices we perform the following: first, draw a uniform index $I \in [1,T']$. then, draw an interval length $\tau \in [1, T'-1]$. We now calculate the second index as the first index drawn added with the interval length drawn, modulo $T'$, i.e. $J = I + \tau \pmod{T'}$. We now name the smaller index of the two $I_1$, and the larger index $I_2$. We have drawn 2 indices in $[1,T']$ uniformly. To get that $I_2 - I_1 \in [T, T'-2T]$, all we need in that the interval drawn, $\tau$, is in the required range: $\tau \in [T, T' - 2T]$. Since $\tau$ is drawn uniformly within the interval $[1,T'-1]$, this happens with probability $Pr(I_2 - I_1 \in [T, T'-2T] | N=2, T') = \frac{T'-3T}{T'-1} \geq \frac{T' - 3T}{T'} = \frac{(K-3)T}{KT}$. Thus, if we take $K \geq \frac{3}{\delta}$, we have that with probability at least $1-\delta$, we spend at least $T$ iterations in each of $M_1$, $M_2$.
\end{proof}

\begin{corollary}\label{corr:enoughIterations_For_N_At_least_2}
Let $T' = KT$, and denote by $\tau_1, \tau_2$ the number of iterations spent in $M_1$, $M_2$, respectively.
Let $\delta \in (0,1)$, then, conditioned on $N \geq 2$, for $K = \frac{3}{\delta}$, with probability at least $1-\delta$, $\tau_1 \geq T, \tau_2 \geq T$.
\end{corollary}
\begin{proof}
We have from lemma~\ref{lemma:enoughIterations}, that $\Pr(\tau_1 \geq T, \tau_2 \geq T | N=2) \geq 1-\delta$, thus:
$\Pr(\tau_1 \geq T, \tau_2 \geq T | N \geq 2) \geq \sum_{n=2}^{T'} \Pr(\tau_1 \geq T, \tau_2 \geq T | N = n) \geq \Pr(\tau_1 \geq T, \tau_2 \geq T | N = 2) \geq 1-\delta$
\end{proof}
We now notice, that we have with high probability (greater than $(1-\delta)$) at least $T$ iterations in each section of the MDP, $M_1$ and $M_2$, given $N \geq 2$. From assumption~\ref{assum:assumptionsRareExperiences}, we have that w.h.p all state-action pairs in the MDP have been visited, and thus after $T'$ iterations, all state-action pairs are stored in memory. We also have that $Pr(N \geq 2) \geq \frac{1}{2}$ for $K \cdot T \geq 20000$. Thus, taking $\delta = 0.03$, and $K = \frac{3}{\delta} = 100$, we have that given that $T' = \max( 20000, 100 \cdot T, \frac{1.73}{p})$, w.p $\geq \frac{1}{2}$, running experience replay for $g^{-1}(\psi)$ iterations on the memory buffer after $T'$ online iterations, when sampling from memory under the stationary distribution induced by the optimal policy, will lead to convergence of the Q values to an $\psi$ proximity to the optimal Q values exactly. 
Note that since we have that: $p < 0.00009$, then $\frac{1.73}{p} < 20000$, and thus, it is enough to have: $T' = \max( 20000, 100 \cdot T)$.
Thus, we conclude that with probability $\geq \frac{1}{2}$, Q-learning with experience replay will converge to $\psi$ proximity to the optimal Q values.

We now show that for this MDP, where some transitions are rare with respect to the time horizon $T'$, we get that running Q-learning without experience replay for $T'$ iterations will lead to having $Q^{T'}_3$ which is far from $Q^\ast_3$, and thus without replay we do not converge to the optimal Q values.

We will consider three cases: first the event that no transitions take place at all, second, the event the there is only one transition, and third, the event that there are two transitions.

Denote the number of updates done on a state-action pair $(s,a)$ as $z(s,a)$. \\
For $N=0$, we have that all iterations were spent in $M_1$, since no transition was performed to $M_2$. Thus, there exists at least one state-action pair $(s,a)$ s.t. $z(s,a) \geq \frac{T'}{|S||A|} \geq \frac{T}{|S||A|}$.

Now, consider the event where $N=1$. In such as case, only one transition was performed. We have that $T' = K*T$ for $K=\frac{3}{0.03}$, thus we can calculate the probability that there are at least $T$ iterations in $M_2$. 
Note that this happens if the iteration of the transition from $M_1$, denoted by $I_1$, was any iteration in the range $[1,T'-T]$. Since we have that all sequences with exactly one transition have the same probability, we get that:
$\Pr(I_1 \in [1, T' - T]) \geq \frac{T' - T}{T'} \geq \frac{(K-1)T}{KT}$. For $\delta = 0.03$, and $K > \frac{3}{\delta}$, then w.p at least $1- \frac{\delta}{3}$, the number of iterations in $M_2$, denoted by $\tau_2$, is at least $T$.
Thus, given that $N=1$ and $\tau_2 \geq T$, we have that there exists at least one state-action pair $(s,a)$, s.t. $z(s,a) \geq \frac{T}{|S||A|}$.

For the third case, we look at the event where the number of transitions is exactly $2$. In this case, after transitioning from $M_1$ to $M_2$, we have at least $T$ iterations in $M_2$, from lemma~\ref{lemma:enoughIterations}. We then transition back to $M_1$. This means, that for all state-action pairs in $M_2$ except $s_2$, no state was updated following the transition back to $M_1$. This means, that the sequence of updates done on any of the states in $M_2$ is indistinguishable from a sequence that was performed when running on,y on $M_2$, i.e., when $\epsilon = 0$. In addition, since there were at least $T$ iterations in $M_2$, there exists at least one state-action pair in $M_2$ for which $z(s,a) \geq \frac{T}{|S||A|}$. We will now compute the probability that no state-action pair except $(s_2,a)$ have $z(s,a) \geq \frac{T}{|S||A|}$. For this to happen, all state action pairs except $s_2$ have at most $\frac{T}{|S||A|} -1$ iterations spent in them. Thus, summing over all of the iterations spent in states other then $s_2$, we get:
\begin{align*}
& \Pr( \nexists s \in S_2\setminus s_2, a \in A: z(s,a) \geq \frac{T}{|S||A|}) \\
& \leq \Pr(z(s_2,a) \geq T - \sum_{s \in S_2\setminus s_2} \sum_{a \in A} \frac{T}{|S||A|} -1) \\
& \leq \Pr(z(s_2,a) \geq T - (|S|-1) \cdot |A| (\frac{T}{|S||A|} -1 )) \\
& \leq \Pr(z(s_2,a) \geq T - \frac{(|S|-1)}{|S|} \cdot T + (|S|-1) \cdot |A| \\
& \leq \Pr(z(s_2,a) \geq \frac{1}{|S|} \cdot T + (|S|-1)\cdot |A|
\end{align*}
We compute the probability for this to happen when $N=2$:
\begin{align*}
& \Pr(z(s_2,a) \geq \frac{1}{|S|} \cdot T + (|S|-1)\cdot |A| | N=2) \\
& \leq \frac{\Pr(N=2 | z(s_2,a) \geq \frac{1}{|S|} \cdot T + (|S| - 1) \cdot |A|) \cdot \Pr(z(s_2,a) \geq \frac{1}{|S|} \cdot T + (|S| - 1) \cdot |A|)}{\Pr(N=2)} \\
& \leq \frac{(\frac{1.73}{T' \cdot p})^2 \cdot (1 - \frac{1.73}{T' \cdot p})^{\frac{1}{|S|} \cdot T + (|S| - 1) \cdot |A| - 2} \cdot \frac{T' \cdot p}{\frac{1}{|S|} \cdot T + (|S| - 1) \cdot |A|}}{0.26531} \\
& \leq \frac{\frac{1.73^2}{T' \cdot p} \cdot (1 - \frac{1.73}{T' \cdot p})^{ T - 2} \cdot \frac{1}{T}}{0.26531}
\end{align*}
Since we have that $T' \geq \max(20000, 100T)$, we have that for any $0.00009 < p < 1$, we get that for any $T' \geq 20000$:
\begin{align*}
\Pr(z(s_2,a) \geq \frac{1}{|S|} \cdot T + (|S|-1)\cdot |A| | N=2) \leq 0.9258
\end{align*} 
And thus:
\begin{align*}
\Pr(\exists s_i \in S_2 \setminus s_2, a_j \in A : z(s_i,a_j) \geq \frac{T}{|S_2||A_2|} | N=2) \geq 0.0742
\end{align*} 
We now look at the Q values of states in $M_2$ for the event of $N=1$ and $N=2$, and the Q-values of states in $M_1$ for the event of $N=0$. When $N=0$, no Q value of any state-action pair in $S_1$ has been updated after observing a transition to the other part of the environment.
For $N=1$ and $N=2$,  no Q value of any state-action pair in $S_2$ (except for $s_2$ when $N=2$) has been updated after observing a transition to the other part of the environment. Thus, all state-action pairs in $S_2$ for $N=1, N=2$ (except for $s_2$ when $N=2$), or in $S_1$ for $N=0$, have the property that their update sequence is indistinguishable from a sequence of updates that is preformed when $\epsilon = 0$, i.e., when running separately on $M_2$ or $M_1$. Thus, we have the following property for such state-action pairs:
Given $N=0, T' \geq T$:
\begin{align*}
&\forall s \in S_1, a \in A : \\
& | Q^{z(s,a)}_3(s,a)-Q^\ast_1(s,a) | \leq f \left( z(s,a) \right) \cdot \left( | Q^0_3(s,a)-Q^\ast_1(s,a) | \right) 
\end{align*}
Given $N=1, \tau_2 \geq T$:
\begin{align*}
&\forall s \in S_2, a \in A : \\
& | Q^{z(s,a)}_3(s,a)-Q^\ast_2(s,a) | \leq f \left( z(s,a) \right) \cdot \left( | Q^0_3(s,a)-Q^\ast_2(s,a) | \right) 
\end{align*}
Given $N=2, \tau_2 \geq T$:
\begin{align*}
&\forall s \in S_2\setminus s_2, a \in A : \\
& | Q^{z(s,a)}_3(s,a)-Q^\ast_2(s,a) | \leq f \left( z(s,a) \right) \cdot \left( | Q^0_3(s,a)-Q^\ast_2(s,a) | \right) 
\end{align*}
Since sequence of updates is indistinguishable from learning in $M_2$ or $M_1$ separately (when $\epsilon = 0$), all these state-action pairs are converging to the optimal Q values of $M_2$ or $M_1$ alone, $Q^\ast_2$ or $Q^\ast_1$, in a rate $f(z(\cdot,\cdot))$.

We now look at the Q values of states in $M_1$ for $N=0$, $M_2$ for $N=1, N=2$.
Denote for every $s \in S_i,a \in A$: $\Delta_i(s,a) = |Q^\ast_i(s,a) - Q^\ast_3(s,a)|$ for $i=1,2$. 

For $N=0$, with probability $1$ there exists at least one state-action pair $(s_i,a_j)$ in $M_1$ for which: $z(s_i,a_j) \geq \frac{T}{|S_1||A_1|}$, and thus for this pair of state-action in $M_1$ we have:
\begin{align*}
& \Pr ( | Q^{z(s_i,a_j)}_3(s_i,a_j)-Q^\ast_1(s_i,a_j) | \leq \frac{D_0}{2} ) \big| N=0, z(s_i,a_j) \geq \frac{T}{|S_1||A_1|} ) \\
& = \Pr ( f\left( \frac{T}{|S_1| |A_1|} \right) \cdot \left( | Q^0_3(s_i,a_j)-Q^\ast_1(s_i,a_j) | \right) \leq \frac{D_0}{2} )| \big| N=0, z(s_i,a_j) \geq \frac{T}{|S_1||A_1|} ) = 1
\end{align*} 
where the last inequality holds since $|Q^0_3(s_i,a_j)-Q^\ast_1(s_i,a_j)| \leq D_0$, and $f\left( \frac{T}{|S_1| |A_1|} \right)$ is monotonically decreasing to $0$ with $f\left( \frac{T}{|S_1| |A_1|} \right) \leq \frac{1}{2}$.
This gives us:
\begin{align*}
&\Pr\left(| Q^{T'}_3(s_i,a_j)-Q^\ast_3(s_i,a_j) | \geq \frac{D_0}{2} \big| N=0, z(s_i,a_j) \geq \frac{T}{|S_1||A_1|} \right) \\
&= \Pr \left(| Q^{z(s_i,a_j)}_3(s_i,a_j)-Q^\ast_3(s_i,a_j) | \geq \frac{D_0}{2} \big| N=0, z(s_i,a_j) \geq \frac{T}{|S_1||A_1|} \right) \\
&= \Pr \left(| Q^{z(s_i,a_j)}_3(s_i,a_j)-Q^\ast_1(s_i,a_j) + Q^\ast_1(s_i,a_j) - Q^\ast_3(s_i,a_j) | \geq \frac{D_0}{2} \big| N=0, z(s_i,a_j) \geq \frac{T}{|S_1||A_1|} \right) \\
& \geq \Pr \left(\left| \left| Q^\ast_1(s_i,a_j) - Q^\ast_3(s_i,a_j) \right| - | Q^{z(s_i,a_j)}_3(s_i,a_j)-Q^\ast_1(s_i,a_j) | \right| \geq \frac{D_0}{2} \big| N=0, z(s_i,a_j) \geq \frac{T}{|S_1||A_1|}  \right) \\
& \geq \Pr\left( \left| \Delta_1(s_i,a_j) - f \left( \frac{T}{|S_1| |A_1|} \right) |Q^0_3(s_i,a_j)-Q^\ast_1(s_i,a_j)| \right| \geq \frac{D_0}{2} \big| N=0, z(s_i,a_j) \geq \frac{T}{|S_1||A_1|} \right) \\
& \geq \Pr\left( D_0 -  f \left( \frac{T}{|S_1| |A_1|} \right) |Q^0_3(s_i,a_j)-Q^\ast_1(s_i,a_j)| \geq \frac{D_0}{2} \big| N=0, z(s_i,a_j) \geq \frac{T}{|S_1||A_1|} \right) \\
&\geq \Pr\left( f \left( \frac{T}{|S_1| |A_1|} \right) |Q^0_3(s_i,a_j)-Q^\ast_1(s_i,a_j)| \leq \frac{D_0}{2} \big| N=0, z(s_i,a_j) \geq \frac{T}{|S_1||A_1|} \right) \\
&=1
\end{align*}
where this stems from the above inequality and $\Delta_1(s_i,a_j) \geq D_0$ from assumption~\ref{assum:assumptionsRareExperiences}.

For $N=1$, with probability $\geq 1-\delta$ we have that the number of iterations in $M_2$ satisfy $\tau_2 \geq T$, and thus there exists at least one state-action pair $(s_i,a_j)$ in $M_2$ for which: $z(s_i,a_j) \geq \frac{T}{|S_2||A_2|}$, and thus for this pair of state-action in $M_2$ we have:
\begin{align*}
&\Pr (| Q^{z(s_i,a_j)}_3(s_i,a_j)-Q^\ast_2(s_i,a_j) | \leq \frac{D_0}{2}  \big| N=1, \tau_2 \geq T, z(s_i,a_j) \geq \frac{T}{|S_2||A_2|} ) \\
&\leq f\left( \frac{T}{|S_2| |A_2|} \right) \cdot \left( | Q^0_3(s_i,a_j)-Q^\ast_2(s_i,a_j) | \right) = 1
\end{align*} 
where the last inequality holds since $|Q^0_3(s_i,a_j)-Q^\ast_2(s_i,a_j)| \leq D_0$, and $f\left( \frac{T}{|S_2| |A_2|} \right)$ is monotonically decreasing to $0$ with $f\left( \frac{T}{|S_2| |A_2|} \right) \leq \frac{1}{2}$.
This gives us:
\begin{align*}
&\Pr\left(| Q^{T'}_3(s_i,a_j)-Q^\ast_3(s_i,a_j) | \geq \frac{D_0}{2} \big| N=1, \tau_2 \geq T, z(s_i,a_j) \geq \frac{T}{|S_2||A_2|} \right) \\
&= \Pr \left(| Q^{z(s_i,a_j)}_3(s_i,a_j)-Q^\ast_3(s_i,a_j) | \geq \frac{D_0}{2} \big| N=1, \tau_2 \geq T, z(s_i,a_j) \geq \frac{T}{|S_2||A_2|} \right) \\
&= \Pr \left(| Q^{z(s_i,a_j)}_3(s_i,a_j)-Q^\ast_2(s_i,a_j) + Q^\ast_2(s_i,a_j) - Q^\ast_3(s_i,a_j) | \geq \frac{D_0}{2} \big| N=1, \tau_2 \geq T, z(s_i,a_j) \geq \frac{T}{|S_2||A_2|} \right) \\
& \geq \Pr \left(\left| \left| Q^\ast_2(s_i,a_j) - Q^\ast_3(s_i,a_j) \right| - | Q^{z(s_i,a_j)}_3(s_i,a_j)-Q^\ast_2(s_i,a_j) | \right| \geq \frac{D_0}{2} \big| N=1, \tau_2 \geq T, z(s_i,a_j) \geq \frac{T}{|S_2||A_2|}  \right) \\
& \geq \Pr\left( \left| \Delta_2(s_i,a_j) - f \left( \frac{T}{|S_2| |A_2|} \right) |Q^0_3(s_i,a_j)-Q^\ast_2(s_i,a_j)| \right| \geq \frac{D_0}{2} \big| N=1, \tau_2 \geq T, z(s_i,a_j) \geq \frac{T}{|S_2||A_2|} \right) \\
& \geq \Pr\left( D_0 -  f \left( \frac{T}{|S_2| |A_2|} \right) |Q^0_3(s_i,a_j)-Q^\ast_2(s_i,a_j)| \geq \frac{D_0}{2} \big| N=1, \tau_2 \geq T, z(s_i,a_j) \geq \frac{T}{|S_2||A_2|} \right) \\
&\geq \Pr\left( f \left( \frac{T}{|S_2| |A_2|} \right) |Q^0_3(s_i,a_j)-Q^\ast_2(s_i,a_j)| \leq \frac{D_0}{2} \big| N=1, \tau_2\geq T z(s_i,a_j) \geq \frac{T}{|S_2||A_2|} \right) \\
&=1
\end{align*}
where this stems from the above inequality and $\Delta_2(s_i,a_j) \geq D_0$ from assumption~\ref{assum:assumptionsRareExperiences}.

For $N=2$, with probability $\geq 1-\delta$ we have that the number of iterations in $M_2$ satisfy $\tau_2 \geq T$, with probability $\geq 0.0742$ there exists at least one state-action pair $(s_i,a_j)$ in $S_2 \setminus s_2$ for which: $z(s_i,a_j) \geq \frac{T}{|S_2||A_2|}$, and thus for this pair of state-action in $M_2$ we have:
\begin{align*}
& \Pr( | Q^{z(s_i,a_j)}_3(s_i,a_j)-Q^\ast_2(s_i,a_j) | \leq \frac{D_0}{2} \big| N=2, \tau_2 \geq T, z(s_i,a_j) \geq \frac{T}{|S_2||A_2|} ) \\
& \leq f\left( \frac{T}{|S_2| |A_2|} \right) \cdot \left( | Q^0_3(s_i,a_j)-Q^\ast_2(s_i,a_j) | \right) = 1
\end{align*} 
where the last inequality holds since $|Q^0_3(s_i,a_j)-Q^\ast_2(s_i,a_j)| \leq D_0$, and $f\left( \frac{T}{|S_2| |A_2|} \right)$ is monotonically decreasing to $0$ with $f\left( \frac{T}{|S_2| |A_2|} \right) \leq \frac{1}{2}$.
Similarly to the case of $N=1$, this gives us:
\begin{align*}
&\Pr\left(| Q^{T'}_3(s_i,a_j)-Q^\ast_3(s_i,a_j) | \geq \frac{D_0}{2} \big| N=2, \tau_2 \geq T, z(s_i,a_j) \geq \frac{T}{|S_2||A_2|} \right) = 1
\end{align*}

We now show that:
\begin{align*}
& \Pr(\exists (s_i,a_j) : | Q^{T'}_3(s_i,a_j)-Q^\ast_3(s_i,a_j) | \geq \frac{D_0}{2}) \geq \frac{1}{2}
\end{align*}
for this we consider the above three cases:
\begin{align*}
& \Pr(\exists (s_i,a_j) : | Q^{T'}_3(s_i,a_j)-Q^\ast_3(s_i,a_j) | \geq \frac{D_0}{2}) \\
&\geq \Pr(\exists (s_i,a_j) : | Q^{T'}_3(s_i,a_j)-Q^\ast_3(s_i,a_j) | \geq \frac{D_0}{2} | N=0, \exists s_i \in S_1, a_j \in A \mbox{ s.t. } z(s_i,a_j) \geq \frac{T}{|S_1||A_1|}) \\
&\cdot \Pr(\exists s_i \in S_1,  a_j \in A \mbox{ s.t. } z(s_i,a_j) \geq \frac{T}{|S_1||A_1|} | N=0) \cdot \Pr(N=0) \\
&+ \Pr(\exists (s_i,a_j) : | Q^{T'}_3(s_i,a_j)-Q^\ast_3(s_i,a_j)| \geq \frac{D_0}{2} | N=1, \tau_2 \geq T, , \exists s_i \in S_2,  a_j \in A \mbox{ s.t. } z(s_i,a_j) \geq \frac{T}{|S_2||A_2|}) \\
&\cdot \Pr(\exists s_i \in S_2,  a_j \in A \mbox{ s.t. } z(s_i,a_j) \geq \frac{T}{|S_2||A_2|} \big| \tau_2 \geq T, N=1) \cdot \Pr(\tau_2 \geq T | N=1) \cdot  \Pr(N=1) \\
& + \Pr(\exists (s_i,a_j) : | Q^{T'}_3(s_i,a_j)-Q^\ast_3(s_i,a_j) | \geq \frac{D_0}{2} | N=2, \tau_2 \geq T, \exists s_i \in S_2 \setminus s_2,  a_j \in A \mbox{ s.t. } z(s_i,a_j) \geq \frac{T}{|S_2||A_2|}) \\
&\cdot \Pr(\exists s_i \in S_2 \setminus s_2, a_j \in A \mbox{ s.t. } z(s_i,a_j) \geq \frac{T}{|S_2||A_2|} | N=2, \tau_2 \geq T) \cdot \Pr(\tau_2 \geq T | N=2) \cdot \Pr(N=2)
\end{align*}
We get:
\begin{align*}
&\Pr( \exists (s_i,a_j) : | Q^{T'}_3(s_i,a_j)-Q^\ast_3(s_i,a_j) | \geq \frac{D_0}{2}) \\
&\geq 1 \cdot 1 \cdot \Pr(N=0) + 1 \cdot 1 \cdot (1-\frac{\delta}{3}) \cdot \Pr(N=1) + 1 \cdot 0.0742 \cdot (1-\delta) \cdot \Pr(N=2)
\end{align*}
Letting $\delta = 0.03$, $T' \geq \max(20000, 100 \cdot T)$, we get:
\begin{align*}
\Pr(\exists (s_i,a_j) : | Q^{T'}_3(s_i,a_j)-Q^\ast_3(s_i,a_j) | \geq \frac{D_0}{2})
\geq  0.177271 + 0.99 \cdot 0.306706 + 0.0742 \cdot 0.97 \cdot 0.26531 \geq \frac{1}{2}
\end{align*}
We got that $\Pr(\exists (s_i,a_j) : | Q^{T'}_3(s_i,a_j)-Q^\ast_3(s_i,a_j) | \geq \frac{D_0}{2}) \geq \frac{1}{2}$, thus, for a large enough $D_0$, the Q values are far from optimal.

\subsection{Asymptotic convergence of Q-learning with experience replay}
\label{app:asymptoticProof}
In this section we prove that algorithm~\ref{alg:Q-learning with replay} converges asymptotically to the optimal policy. This proof is similar in structure to the proof of Theorem~\ref{theorem:convergenceRate}, however, it is simpler and thus we present it here.

We assume the following assumption:
\begin{assumption}
\label{assum:ConditionsForLimitConvergence}
For every state $s \in S$, and every action $a \in A$, the pair $(s,a)$ is visited infinitely many times.
\end{assumption}
We assume that each state and action are visited infinitely many times, as formally stated in Assumption~\ref{assum:ConditionsForLimitConvergence}. This is similar to standard assumptions such as GLIE (greedy in the limit with infinite exploration, \cite{singh2000convergence}), and is important for ensuring sufficient exploration.

\begin{theorem}
\label{theorm:Convergence in the limit}
Let $M$ be an MDP with a state set $S$, an action set $A$, and an optimal Q-function $Q^{\ast}$. 
Let $\gamma \in (0,1)$, and for each $s \in S, a \in A$ the following conditions hold for $\alpha_t(s,a)$:
\begin{itemize}
\item $\forall t: \alpha_t(s,a) \geq 0$
\item $\sum_{t=0}^\infty \alpha_t(s,a) = \infty$
\item $\sum_{t=0}^\infty \alpha_t(s,a)^2 < \infty$
\end{itemize}
Then, under assumption~\ref{assum:ConditionsForLimitConvergence},
Algorithm~\ref{alg:Q-learning with replay} converges almost surely in the limit to the optimal Q-function $Q^{\ast}$.
\end{theorem}

We use the framework of iterative stochastic algorithms \citep{bertsekas1996neuro} as a basis for our convergence proof.
For completeness, we start by showing how Q-learning with experience replay can be written as an iterative stochastic algorithm. This is similar to the formulation in \cite{bertsekas1996neuro} and \cite{even2003learning}.

\subsubsection{Iterative Stochastic Algorithms}
An iterative stochastic algorithm is an algorithm of the form: \\
$X_{t+1}(i) = (1-\alpha_t(i)) \cdot X_t(i) + \alpha_t(i) \cdot \left( H_tX_t(i) + w_t(i) \right)$
where $H_t$ is a pseudo-contraction mapping, selected from a family of mappings $\mathcal{H}$, where the choice is a function of information contained in the past history $\mathcal{F}_t$. 

\paragraph{Contraction and Pseudo Contraction Mappings}
A function $H: \mathbb{R}^n \rightarrow \mathbb{R}^n$ is a maximum norm pseudo contraction if there exists some $x^\ast \in \mathbb{R}^n$ and a constant $\beta \in [0,1)$ such that
\begin{gather*}
\forall x: \| Hx - x^\ast \|_\infty \leq \beta \cdot \| x - x^\ast \|_\infty
\end{gather*}

If $H$ is a maximum norm pseudo contraction, then $x^\ast$ is a fixed point of $H$, furthermore, it is the only fixed point of $H$.

$H$ is a contraction mapping if it satisfies $\forall x,\bar{x}: \|Hx - H\bar{x} \|_\infty \leq \beta \|x-\bar{x} \|_\infty $

If $H$ is a contraction mapping, then it is automatically a pseudo contraction mapping.

\subsubsection{Q-Learning with Experience Replay as an Iterative Stochastic Algorithm}

We briefly show how Q-learning with experience replay can be written as an iterative stochastic algorithm.
We use the contraction mapping:
\begin{align*}
HQ(s,a) = \sum_{s' \in S} p_{s,s'}(a) \left( R(s,a) + \gamma max_{a'} Q(s',a') \right)
\end{align*}
and the noise term:
\begin{align*}
w_t(s_t,a_t) &= r(s_t,a_t) + \gamma \cdot max_{a'}Q_t(s_{t+1},a') - \sum_{s' \in S} p_{s,s'}(a_t) \cdot \left( R(s_t,a_t) + \gamma \cdot max_{a'} Q(s',a') \right) \\
&= r(s_t,a_t) + \gamma \cdot max_{a'}Q_t(s_{t+1},a') - HQ(s_t,a_t)
\end{align*}

where $R(s,a) = \mathbb{E}[r(s,a)]$. Combining the two we get:
\begin{align*}
&Q_{t+1}(s_t,a_t) = (1-\alpha_t(s_t,a_t)) \cdot Q_t(s_t,a_t) + \alpha_t(s_t,a_t) \cdot \left( HQ_t(s_t,a_t) + w_t(s_t,a_t) \right)
\end{align*}
which is exactly the update rule in Q-learning. 

A formal proof showing that $H$ is a contraction mapping is found in Appendix~\ref{app:IterativeStocahsticAlg}.

\subsubsection{Q-learning With Experience Replay Converges To The Optimal Policy}
\label{subsec:asymptoticProof}
We show that Q-learning with experience replay as described in algorithm~\ref{alg:Q-learning with replay} converges in the limit to the optimal decision-making policy, $Q^{\ast}$.
The proof follows the general framework of convergence for iterative stochastic algorithms - we start by showing that initially, the Q value for all state-action pairs is within some constant distance from the optimal Q values. We then show that the Q values can be decomposed to 2 parts, one which describes the deterministic process of convergence if all updates were done using the expected value of all random variables, and a second part that holds the error accumulated by the stochastic sampling of the random variables in question. Using this decomposition, we show that we can bound the distance between the Q values and the optimal Q values by smaller and smaller constants, which leads to convergence. This is a common technique used to prove convergence in iterative stochastic algorithms, however, since we are sampling both from the MDP and from a memory buffer, in our case the error accumulated by the sampling of memories does not have an expected value of $0$. This is the main challenge in proving the convergence of Q-learning with experience replay, and is the heart of our proof. 

Let $\alpha_t(s,a)$ be the learning rate parameter at iteration t for state s and action a. We take $\alpha_t(s,a) = 0$ for state-action pairs that are not being updated in iteration $t$. It is important to note that $\alpha_t(s,a)$ is in fact a random variable that depends on the randomness of the environment in past iterations and on the randomness of the algorithm.

We denote the history (all transitions, rewards and actions before iteration $t$) as $F_t$, i.e. $F_t = \{ s_1,a_1,r_1,s_2,a_2,r_2,...,s_{t-1},a_{t-1},r_{t-1} \}$. 
We next study the effect of the noise term $w_t$ on the convergence of the algorithm. For iterations of online Q-learning we have that the expected value of the noise term, $w$, is $0$.
\begin{align*}
\mathbb{E}\left[ w_t(s,a) | F_t \right] &=\mathbb{E} [ r(s,a) + \gamma \cdot max_{a'} Q_t(s_{t+1},a') \\
&- \sum_{s' \in S} p_{s,s'}(a) \cdot \left( R(s,a) + \gamma \cdot max_{a'} Q_t(s',a') \right) | F_t ] \\
&= \mathbb{E} \left[ r(s,a) \right] - R(s,a) \\
&+ \gamma \cdot \sum_{s' \in S} p_{s,s'}(a) max_{a'}Q_t(s',a') - \gamma \cdot \sum_{s' \in S} p_{s,s'}(a) \cdot max_{a'} Q_t(s',a')\\
 &= 0
\end{align*}

On the other hand, for iterations where we sample from the memory pool, we do not have that $\mathbb{E} \left[ w_t(s,a) | F_t \right] = 0$, since in these iterations the expectation of $r(s,a)$ and $s_{t+1}$ are not equal to the true expectations as in the markov chain, but to the empirical means of the experiences collected in the memory buffer. 
Thus, in iterations where we sample from the memory, we have that:
\begin{align*}
\mathbb{E} \left[ w_t(s,a) | F_t \right] &= \mathbb{E} \left[ r(s,a) + \gamma max_{a'} Q_t(s_{t+1},a') | F_t \right] - R(s,a) - \gamma \sum_{s'=0}^n p_{s,s'}(a) max_{a'} Q_t(s',a') \\
&=\widehat{r(s,a)} - R(s,a) + \gamma \cdot \widehat{max_{a'}Q_t(s_{t+1},a')} - \gamma \cdot \sum_{s'=0}^n p_{s,s'}(a) \cdot max_{a'} Q_t(s',a')
\end{align*}

where $\widehat{r(s,a)}$ and $\widehat{max_{a'}Q_t(s_{t+1},a')}$ are the empirical means of the reward in state-action pair $s,a$ and the empirical maximal $Q$ value for the next state $s_{t+1}$ (averaged over the experiences in the memory buffer).

Explicitly, mark as $L$ all transitions in memory buffer from state $s$ with action $a$, and let a transition be a tupple $[s_1,a,s_2,r]$, then we can write these empirical means the following way:
\begin{align*}
&\widehat{r(s,a)} = \frac{1}{|L|} \sum_{r \in L} r \\
&\widehat{max_{a'}Q_t(s_{t+1},a')} = \frac{1}{|L|} \cdot \sum_{s_2 \in L} max_{a'}Q_t(s_2,a')
\end{align*}



Note that since we have $|r(\cdot,\cdot)| < R_{max}$, then $|w(\cdot,\cdot)| \leq \sum_{i=0}^\infty \gamma^i \cdot R_{max}\leq \frac{R_{max}}{1-\gamma}$, and thus $\forall s,a,t : \mathbb{E}\left[ w_t^2(s,a) |F_t  \right] \leq C$ where C is a constant. \\

We now use the following lemmas in order to show that Q-learning with experience replay converges.
\boundedSeq*

\noiseTerm*

Proofs appear in appendices~\ref{app:boundedSeqProof},~\ref{app:proofLemmaNoiseTermAsymptotic}.
Lemma~\ref{lemma:boundedSeq} states there exists a constant such that all Q values for all $t$ have a bounded distance to the optimal Q values. 

Lemma~\ref{lemma:noiseTerm} ensures that even though we do not always sample from the MDP, we still have that the accumulated noise converges to $0$. We give a sketch of the proof here, and  give the formal proof in appendix~\ref{app:proofLemmaNoiseTermAsymptotic}.
To prove lemma~\ref{lemma:noiseTerm}, we define: $V_t(s,a) = W_t(s,a) - \mathbb{E} \left[ W_t(s,a) \right]$, and show that in the limit, both $V_t(s,a)$ and $\mathbb{E} \left[ W_t(s,a) \right]$ converge to $0$, and thus $W_t(s,a)$ converges to $0$ as well. 
To show $V_t(s,a)$ we write it as a stochastic iterative algorithm where the noise term has expectation $0$, and thus it converges. We then turn to show  $E[W_t] \rightarrow 0$. We look at the expected value of $w_t(s,a)$. Following assumption~\ref{assum:ConditionsForLimitConvergence}, we have that the number of samples of each state-action pair in the memory buffer continues to grow with $t$, and thus the empirical mean of the reward and transition probabilities from $s,a$ in the buffer converge to the real mean. This lead to having $E[w_t(s,a)]$ converge to $0$. We then look at the expectation of the accumulated noise $E[W_t(s,a)]$, and prove that there are infinitely many time points in which this expected value is smaller than some $\epsilon$. We take one such time point, and show that this will remain true for all following time points. This gives us the convergence of $E[W_t]$ to $0$.
  


For the proof of Theorem~\ref{theorm:Convergence in the limit}, we write $Q(s,a)$ as a sum of two components: $Y_t(s,a)$ and $W_t(s,a)$, which are specified in the proof of lemma~\ref{lemma:convergenceInLimit}. 
We define $\forall \tau,s,a: W_{\tau:\tau}(s,a) = 0$ and $W_{t+1:\tau}(s,a) = (1-\alpha_t(s,a))W_{t:\tau}(s,a) + \alpha_t(s,a)w_t(s,a)$.
We also define $D_{k+1} = (\gamma + \epsilon)\cdot D_k$, for $\epsilon = \frac{1 - \gamma}{2}$, for all $k \geq 0$.

We prove Theorem~\ref{theorm:Convergence in the limit} (namely, that $lim_ {t\rightarrow \infty} \| Q_{t}(s,a) - Q^\ast(s,a) \| = 0$) by showing the lemma~\ref{lemma:convergenceInLimit}, which states: \\
Let $\gamma \in (0,1)$. For each $k$, there exists some time $t_k$ s.t. for all $t \geq t_k$, $\| Q_{t}(s,a) - Q^\ast(s,a) \| \leq D_k$.

The full proof of of lemma~\ref{lemma:convergenceInLimit} is found in appendix~\ref{app:convergenceInLimitProof}.
The proof of lemma~\ref{lemma:convergenceInLimit} follows by induction.
First we notice that since $\gamma < 1$, $lim_{k \rightarrow \infty} D_k = 0$. 
We assume that there exists some $t_k$ such that $\forall t \geq t_k : \| Q_{t}(s,a) - Q^\ast(s,a) \| \leq D_k$, and show that there exists some $t_{k+1}$ for which
$\forall t \geq t_{k+1} : \| Q_{t}(s,a) - Q^\ast(s,a) \| \leq D_{k+1}$. \\
For $t_0=0$, we have that $\forall t > t_0 : \| Q_{t}(s,a) - Q^\ast(s,a) \| \leq D_0$, from Lemma~\ref{lemma:boundedSeq}, and thus the induction basis holds. \\
We define $Y_{t_k}(s,a) = D_k$ for all $s \in S ,a \in A$, and $\forall t \geq t_k : Y_{t+1}(s,a) = (1-\alpha_t(s,a))Y_t(s,a) + \alpha_t(s,a) \gamma D_k$. We also define $\forall \tau,s,a: W_{\tau:\tau}(s,a) = 0$ and $W_{t+1:\tau}(s,a) = (1-\alpha_t(s,a))W_{t:\tau}(s,a) + \alpha_t(s,a)w_t(s,a)$.

For every $s,a, t \geq t_k$, we show that the Q values are bounded:
\begin{align*}
\begin{split}
-Y_t(s,a) - W_{t:t_k}(s,a) \leq Q_{t}(s,a) - Q^\ast(s,a) \leq Y_t(s,a) + W_{t:t_k}(s,a)
\end{split}
\end{align*}
We then show that $lim_{t \rightarrow \infty} Y_t(s,a) = \gamma \cdot D_k < D_k$, and combine this with the results of lemma~\ref{lemma:noiseTerm} that states that $lim_{t \rightarrow \infty} W_t(s,a) = 0$. This gives us the required result.

\subsection{Q-learning as an iterative stochastic algorithm}
\label{app:IterativeStocahsticAlg}
We show that the defined $H$ is indeed a contraction mapping for every $t$, thus we will conclude that $H$ is also a pseudo-contraction mapping for every $t$. 
Our mapping $H$ is the following:
\begin{align*}
HQ(s,a) = \sum_{s' \in S} p_{s,s'}(a) \left( R(s,a) + \gamma max_{a'} Q(s',a') \right)
\end{align*}

Let $Q, \bar{Q}$ be two Q-value tables. We will show: $\| HQ - H \bar{Q} \|_\infty \leq \gamma \cdot \| Q - \bar{Q} \|_\infty $. For each $s,a$, we have:
\begin{align*}
&| HQ(s,a) - H \bar{Q}(s,a) | \\
&= \left| \sum_{s'\in S} p_{s,s'}(a)  \gamma \left( max_{a'}Q(s',a') - max_{a'}\bar{Q}(s',a') \right) \right| \\
&\leq \sum_{s'\in S} p_{s,s'}(a) \cdot \gamma \cdot |max_{a'}(Q(s',a')) - \bar{Q}(s',a') | \\
&\leq \sum_{s'\in S} p_{s,s'}(a) \cdot \gamma \cdot max_{a'}\left(|Q(s',a') - \bar{Q}(s',a') | \right) \\ 
&\leq \sum_{s'\in S} p_{s,s'}(a)  \gamma max_{s'}\left( max_{a'} \left( |Q(s',a') - \bar{Q}(s',a') |\right) \right) \\
&\leq \gamma \cdot max_{s'}\left( max_{a'} \left( |Q(s',a') - \bar{Q}(s',a') |\right) \right) \\
&\leq \gamma \cdot \| Q - \bar{Q} \|_\infty
\end{align*}
This is true for any $s,a$ and thus $H$ is a contraction mapping, thus a pseudo-contraction mapping.

\subsection{Proof of Lemma~\ref{lemma:convergenceInLimit}}
\label{app:convergenceInLimitProof}
\begin{proof}
The proof follows by induction.
First we notice that since $\gamma < 1$, $lim_{k \rightarrow \infty} D_k = 0$.
For $t_0=0$, we have that $\forall t > t_0 : \| Q_{t}(s,a) - Q^\ast(s,a) \| \leq D_0$, from Lemma~\ref{lemma:boundedSeq}, and thus the induction basis holds. \\
We now assume that there exists some $t_k$ such that $\forall t \geq t_k : \| Q_{t}(s,a) - Q^\ast(s,a) \| \leq D_k$, and show that there exists some $t_{k+1}$ for which
$\forall t \geq t_{k+1} : \| Q_{t}(s,a) - Q^\ast(s,a) \| \leq D_{k+1}$ \\

Let $Y_{t_k}(s,a) = D_k$ for all $s \in S ,a \in A$, and define $\forall t \geq t_k : Y_{t+1}(s,a) = (1-\alpha_t(s,a))Y_t(s,a) + \alpha_t(s,a) \gamma D_k$ 

For every $s,a, t \geq t_k$, we show that:
\begin{align*}
&-Y_t(s,a) - W_{t:t_k}(s,a)  \leq Q_{t}(s,a) - Q^\ast(s,a) \leq Y_t(s,a) + W_{t:t_k}(s,a)
\end{align*}
We show this by induction: 
For $t=t_k$, we have that $Y_{t_k}(s,a) = D_k$ and $W_{t_k:t_k} = 0$, and thus the induction basis holds. 

For $t>t_k$ we have that since $H$ is a pseudo-contraction mapping, $\|HQ_t - Q^\ast\|_\infty \leq \gamma \cdot \| Q_t - Q^\ast \|_\infty$.

We look at $t=t_k +1$, we have:
\begin{align*}
&Q_t(s,a) - Q^\ast(s,a) \\
&= (1-\alpha_{t-1}(s,a)) Q_{t-1}(s,a) +  \alpha_{t-1}(s,a) \cdot HQ_t(s,a) \alpha_{t-1}(s,a) \cdot w_{t-1}(s,a)  - Q^\ast(s,a) \\ &\leq (1-\alpha_{t-1}(s,a)) (Y_{t-1}(s,a) + W_{t-1:t_k}(i)) + \alpha_{t-1}(s,a) HQ_{t-1}(s,a) + \alpha_{t-1}(s,a) w_{t-1}(s,a) - Q^\ast(s,a) \\ 
&\leq (1-\alpha_{t-1}(s,a)) (Y_{t-1}(s,a) + W_{t-1:t_k}(i)) + \alpha_{t-1}(s,a) \gamma \| Q_{t-1} - Q^\ast \|_\infty + \alpha_{t-1}(s,a) w_{t-1}(s,a) \\
&\leq (1-\alpha_{t-1}(s,a)) (Y_{t-1}(s,a) + W_{t-1:t_k}(i)) + \alpha_{t-1}(s,a) \gamma D_k  + \alpha_{t-1}(s,a) w_{t-1}(s,a) \\
& = Y_t(s,a) + W_{t:t_k}(s,a)
\end{align*}

A symmetrical argument gives
\begin{align*}
Q_t(s,a) - Q^\ast(s,a) \geq -Y_t(s,a) - W_{t:t_k}(s,a)
\end{align*}
We can continue on the same argument by induction as above for any $t \geq t_k$, and thus, we have obtained that for every $t\geq t_k$ we have:
\begin{align*}
-Y_t(s,a) - W_{t:t_k}(s,a) \leq Q_{t}(s,a) - Q^\ast(s,a) \leq Y_t(s,a) + W_{t:t_k}(s,a)
\end{align*}

We finalize the proof by proving that $lim_{t \rightarrow \infty} Y_t(s,a) = lim_{t \rightarrow \infty} (1-\alpha_{t-1}(s,a))\cdot Y_{t-1}(s,a) + \alpha_{t-1}(s,a)\gamma D_k = \gamma \cdot D_k $, and combining this with the results of lemma~\ref{lemma:noiseTerm} that states that $lim_{t \rightarrow \infty} W_t(s,a) = 0$.

To prove $lim_{t \rightarrow \infty} Y_t(s,a) = \gamma D_k$, we look at the following difference:
\begin{align*}
&|Y_{t+1}(s,a) - \gamma D_K| \\
&= |(1- \alpha_t(s,a)) Y_{t}(s,a) + \alpha_t(s,a) \gamma D_k - \gamma D_k| \\
& = |(1-\alpha_t(s,a)) \cdot \left( Y_{t}(s,a) - \gamma D_k \right) | \\
& = ... = \left( \prod_{i=t_k}^t (1-\alpha_i(s,a)) \right) \cdot | Y_{t_k}(s,a) - \gamma D_k |
\end{align*}
And so, to show $lim_{t \rightarrow \infty} |Y_t(s,a) -\gamma D_k| = 0$, it is enough to show: $lim_{t\rightarrow \infty}  \prod_{i=t_k}^t (1-\alpha_i(s,a)) = 0$. To show this, we look at: 
$\log(\prod_{i=t_k}^t (1-\alpha_i(s,a))) = \sum_{i=t_k}^t \log((1-\alpha_i(s,a))) \leq \sum_{i=t_k}^t -\alpha_i(s,a)$, and since $lim_{t \rightarrow \infty} \sum_{i=t_k}^t -\alpha_i(s,a) = - \infty$, we get the result.

We combine the above with Lemma~\ref{lemma:noiseTerm} that states $lim_{t \rightarrow \infty} W_{t:t_k}(s,a) = 0$, and get that 
\begin{align*}
lim_{t \rightarrow \infty} \| Q_t - Q^\ast \|_\infty \leq lim_{t \rightarrow \infty} Y_t + W_{t:t_k} \leq \gamma D_k < D_{k+1}
\end{align*}

Thus, there exists some $t_{k+1}$ such that for $t \geq t_{k+1} : \| Q_t - Q^\ast \|_\infty \leq D_{k+1}$ 
\end{proof}

\subsection{Proof of Lemma~\ref{lemma:deterministicPartConvergenceTime}}
\label{app:proofYconvergenceTime}
\begin{proof}
Starting from $Y_{t_k}(s,a) = D_k$, we have:
\begin{align*}
Y_{t_k+1}(s,a) &= (1-\alpha_{t_k}(s,a))Y_{t_k}(s,a) + \alpha_{t_k}(s,a) \gamma D_k \\
&=(1-\alpha_{t_k}(s,a))(1-\gamma)D_k + (1-\alpha_{t_k}(s,a))\gamma D_k + \alpha_{t_k}(s,a) \gamma D_k \\
&=(1-\alpha_{t_k}(s,a))(1-\gamma)D_k + \gamma D_k \\
Y_{t_k + 2} (s,a) &= (1-\alpha_{t_k + 1}(s,a)) \cdot \left((1-\alpha_{t_k}(s,a))(1-\gamma)D_k + \gamma D_k \right) + \alpha_{t_k + 1}(s,a) \gamma D_k \\
&= (1-\alpha_{t_k + 1}(s,a)) \cdot (1-\alpha_{t_k}(s,a))(1-\gamma)D_k + \gamma D_k \\
\end{align*}
For a general t, we get the following expression:
\begin{align*}
Y_t (s,a) &= \gamma D_k + (1- \gamma) D_k \prod_{i=t_k}^{t-1} (1-\alpha_i(s,a))
\end{align*}
Remember we are looking at $\alpha_t=\frac{1}{t}$ and so we have:
\begin{align*}
Y_t (s,a) &= \gamma D_k + (1- \gamma) D_k \prod_{i=t_k}^{t-1} (1-\frac{1}{i}) 
\end{align*}

We would like to find a t for which:
\begin{align*}
\gamma D_k + (1- \gamma) D_k \prod_{i=t_k}^{t-1} (1-\frac{1}{i}) \leq (\gamma + \frac{2\epsilon}{3}) D_k
\end{align*}
Since $\epsilon = \frac{1 - \gamma}{2}$, we find a t for which:
\begin{align*}
\gamma D_k + (1- \gamma) D_k \prod_{i=t_k}^{t-1} (1-\frac{1}{i}) \leq (\gamma + \frac{1-\gamma}{3}) D_k
\end{align*}
Simplifying the equation, we have:
\begin{align*}
\prod_{i=t_k}^{t-1} (1-\frac{1}{i}) \leq \frac{1}{3}
\end{align*}
Note that the left hand side is a telescopic product, thus we get:
\begin{align*}
\frac{t_k -1}{t - 1} \leq \frac{1}{3}
\end{align*}
which gives the required $t$: $t \geq 3 t_k$
\end{proof}

\subsection{Proof of Lemma~\ref{lemma:boundedSeq}}
\label{app:boundedSeqProof}
\begin{proof}
We remind ourselves the update rule: $Q_{t+1}(s,a) = (1-\alpha_t(s,a)) \cdot Q_{t}(s,a) + \alpha_t(s,a) \cdot \left( r_t + \gamma \cdot max_{a'} Q_t(s_{t+1},a') \right)$

We also know that $\forall t: |r_t| < C$ for some constant $C$.

For any policy $Q$, we can define the value of a state-action pair by the policy as the infinite discounted rewards collected by the policy when starting from state s and action a, i.e:
$V(s,a) = \sum_{t=0}^\infty \gamma^t \cdot r_i$, where $r_i$ is the reward received at trial i after, where the action chosen is the action maximizing the reward according to policy Q.
We then continue on to see that since $|r| \leq C$, then the value of any state-action pair, by any policy chosen, is bounded by:
$V(s,a) = \sum_{t=0}^\infty \gamma^t \cdot r_i \leq \sum_{t=0}^\infty \gamma^t \cdot C \leq \frac{C}{1-\gamma} := V_{max}$

Finally, we show that for any $t,s,a$, $|Q_t(s,a)| \leq max(Q_0(s,a), V_{max})$. 

First, we remind ourselves that for any $t$, $\alpha_t \in [0,1]$. 
Now, suppose that $\forall s,a: |Q_0(s,a)| \leq V_{max}$. 
Suppose there exists some $t$ (could be $t=0$) for which $\forall s,a: |Q_t(s,a)| \leq V_{max} $, then, in $t+1$ we have:
\begin{align*}
& Q_{t+1}(s,a) = (1-\alpha_t)Q_{t}(s,a) + \alpha_t \cdot \left( r_t + \gamma \cdot max_{a'} Q_t(s_{t+1},a') \right) \\
& \leq (1-\alpha_t) \cdot V_{max} + \alpha_t \cdot (C + \gamma \cdot V_{max}) \\
& \leq (1-\alpha_t) \cdot V_{max} + \alpha_t \cdot (C + \gamma \cdot \frac{C}{1-\gamma}) = (1-\alpha_t) \cdot V_{max} + \alpha_t \cdot V_{max} = V_{max}
\end{align*}
A symmetric argument shows that $Q_{t+1}(s,a) \geq -V_{max}$.\\
We get that if at iteration $t$, $\| Q_t \|_\infty \leq V_{max}$ then this is also true in iteration $t+1$. Now, since we started with $Q_0(s,a) \leq V_{max}$ for all $s,a$, this result holds for the entire sequence, and $Q_t$ is bounded.

For the second scenario, if $Q_0(s,a) > V_{max}$, then we have the following result:
Suppose there exists some $t$ (could be $t=0$) for which $\| Q_t \|_\infty \geq V_{max} $, and mark it by $\| Q_t \|_\infty = Q_{max}$, then, in $t+1$ we have:
\begin{align*}
& Q_{t+1}(s,a) = (1-\alpha_t) \cdot Q_t(s,a) + \alpha_t \cdot \left( r_t + \gamma \cdot max_{a'} Q_t(s_{t+1},a') \right) \\
& \leq (1-\alpha_t) \cdot Q_{max} + \alpha_t \cdot \left( C + \gamma \cdot Q_{max} \right) \\
& \leq (1 - \alpha_t + \alpha_t \cdot \gamma ) Q_{max} + \alpha_t \cdot C \leq Q_{max}
\end{align*}
where the last inequality stems from the fact that $Q_{max} \geq \frac{C}{1-\gamma} \geq C$, and $\gamma \in (0,1)$.
A symmetric argument gives the bound from below, and thus we get that if we start with $\|Q_0\|_\infty \geq V_{max}$, then for all consecutive iterations, $Q_t(s,a)$ will be bounded by the initial value.
We conclude that for all $t$, $\|Q_t\|_\infty \leq max(\| Q_0 \|_\infty,V_{max})$ and so there exists a constant $D_0$ (that depends on $Q^\ast$) such that  $\| Q_t - Q^\ast\| \leq D_0 $. Since $\|Q^\ast\|_\infty \leq V_{max}$, we get that $D_0 \leq V_{max} + max(\| Q_0 \|_\infty,V_{max})$.

\end{proof}

\subsection{Proof of Lemma~\ref{lemma:noiseTerm}}
\label{app:proofLemmaNoiseTermAsymptotic}
\begin{proof}
First, we remind ourselves that for iterations of online Q-learning (and not replay iterations), we have that $\mathbb{E} \left[w_t(s,a) | F_t \right] = 0$. For iterations of the experience replay, we have by the strong law of large numbers that:
$\widehat{r(s,a)} \rightarrow R(s,a)$ as $|L| \rightarrow \infty$, and $\widehat{max_{a'}Q_t(s_{t+1},a')} \rightarrow  \sum_{s'=0}^n p_{s,s'}(a) \cdot max_{a'} Q_t(s',a')$ as $|L| \rightarrow \infty$, where $L$ are all transitions in the memory buffer from state s with action a. Note that from the assumption that all state-action pairs are visited infinitely many times (asuumption~\ref{assum:ConditionsForConvergence} for the finite time analysis, asuumption~\ref{assum:ConditionsForLimitConvergence} for the asymptotic convergence analysis), $|L| \rightarrow \infty$ as $t \rightarrow \infty$. \\
This implies that $lim_{t \rightarrow \infty} \mathbb{E} \left[ w_t(s,a) \right] = 0$, where convergence here is almost surely. \\
As defined previously, we have: 
\begin{equation}
\label{eq:W_t_def}
W_{t+1}(s,a) = (1-\alpha_t(s,a))W_{t}(s,a) + \alpha_t(s,a)w_t(s,a)
\end{equation}
and so: 
\begin{equation}
\label{eq:expected_W}
\mathbb{E} \left[ W_{t+1}(s,a) \right] = (1-\alpha_t(s,a)) \mathbb{E} \left[ W_{t}(s,a) \right] + \alpha_t(s,a) \mathbb{E} \left[w_t(s,a) \right]
\end{equation}
Define:
\begin{align}
&V_t(s,a) = W_t(s,a) - \mathbb{E} \left[ W_t(s,a) \right] \\
&g_t(s,a) = w_t(s,a) - \mathbb{E} \left[ w_t(s,a) \right]
\end{align}
Subtracting the equations~\ref{eq:W_t_def} and~\ref{eq:expected_W} we get:
\begin{align*}
V_{t+1}(s,a) = (1-\alpha_t(s,a)) V_t(s,a) + \alpha_t(s,a) g_t(s,a)
\end{align*}

Note that $\mathbb{E}[g_t(s,a)] = 0$, and $\sum_{t=0}^\infty \alpha_t(i) = \infty$ and $\sum_{t=0}^\infty \alpha_t^2(i) < \infty$. Thus, $lim_{t \rightarrow \infty} V_t(s,a) = 0$ (as shown in \cite{bertsekas1996neuro}).

We now look at $\mathbb{E}[ W_{t+1}(s,a)] = \mathbb{E}[(1-\alpha_t(s,a)) W_t(s,a) + \alpha_t w_t(s,a)]$, and show that $\mathbb{E}[W_t(s,a)] \rightarrow 0$ as $t \rightarrow \infty$.
We have from before that $\mathbb{E}[w_t(s,a)] \rightarrow 0$ as $t \rightarrow \infty$.
Given any $\epsilon > 0$, we claim that $| \mathbb{E}[W_t(s,a)]| < \epsilon$ for infinitely many $t$. Otherwise, assume by contradiction that there is a finite number of time points for which this holds. This means, that there exists some $\bar{t_1}$, such that for all $t \geq \bar{t_1}$, $| \mathbb{E}[W_t(s,a)]| > \epsilon$. We also have that there exists some $\bar{t_2}$ such that for all $t \geq \bar{t_2}$, $|\mathbb{E}[w_t(s,a)]| < \frac{\epsilon}{2}$. Take $\bar{t} = max(\bar{t_1}, \bar{t_2})$. Then for all $t \geq \bar{t}$ we have that:
\begin{align*}
 |\mathbb{E}[W_{t+1}(s,a)]| &\leq (1-\alpha_t(s,a)) \cdot |\mathbb{E}[W_t(s,a)]| + \alpha_t(s,a) \cdot |\mathbb{E}[w_t(s,a)]| \\
& \leq |\mathbb{E}[W_t(s,a)]| - \alpha_t(s,a) \cdot |\mathbb{E}[W_t(s,a)]| + \alpha_t(s,a) \cdot \frac{\epsilon}{2} \\
& \leq |\mathbb{E}[W_t(s,a)]| - \alpha_t(s,a) \cdot \epsilon + \alpha_t(s,a) \cdot \frac{\epsilon}{2} \\
& \leq |\mathbb{E}[W_t(s,a)]| - \alpha_t(s,a) \cdot \frac{\epsilon}{2}
\end{align*}

Note that this is true for all $t \geq \bar{t}$. We now have that for all $m \geq \bar{t}$, we have:
\begin{align*}
0 \leq |\mathbb{E}[W_{m+1}(s,a)]| \leq |\mathbb{E}[W_{\bar{t}}(s,a)]| - \sum_{i = \bar{t}}^m \alpha_i(s,a) \cdot\frac{\epsilon}{2}
\end{align*}
however, for all $s,a$, we have that $\sum_{j=0}^\infty \alpha_j(s,a) \rightarrow \infty$, and so this is true also for $\sum_{j=\bar{t}}^\infty \alpha_j(s,a) \rightarrow \infty$. Thus, we get that when $m \rightarrow \infty$, $|\mathbb{E}[W_{m}(s,a)]| \rightarrow -\infty$, in contradiction to the fact that $|\mathbb{E}[W_{m}(s,a)]| \geq 0$. Thus, we have that for all $\epsilon$, $| \mathbb{E}[W_t(s,a)]| < \epsilon$ for infinitely many $t$.

Now, suppose that for some $t \geq \bar{t_2}$, $| \mathbb{E}[W_t(s,a)]| < \epsilon$. Then we have:
\begin{align*}
| \mathbb{E}[W_{t+1}(s,a)]| &< (1-\alpha_t(s,a)) \cdot \epsilon + \alpha_t(s,a) \cdot |\mathbb{E}[w_t(s,a)]| \\
&\leq (1-\alpha_t(s,a)) \cdot \epsilon + \alpha_t(s,a) \cdot \frac{\epsilon}{2} < \epsilon
\end{align*}

By induction, this argument holds for all consecutive iterations, and thus we have shown that $lim_{t \rightarrow \infty} \mathbb{E}[W_t(s,a)] \rightarrow 0$.

Remember that $V_t(s,a) = W_t(s,a) - \mathbb{E}[W_t(s,a)]$, and that $lim_{t \rightarrow \infty} V_t(s,a) \rightarrow 0$, combining with $lim_{t \rightarrow \infty} \mathbb{E}[W_t(s,a)] \rightarrow 0$, we get that: \\
$lim_{t \rightarrow \infty} W_t(s,a) \rightarrow 0$ almost surely, as required.

\end{proof}

\subsection{Proof of Lemma~\ref{lemma:meanOfUnknownSampleSize}}
\label{app:meanOfUnknownSampleSize}
\begin{proof}
Let $r_1,r_2,...,r_t$ be a sequence of i.i.d random variables. 
By Hoeffding's inequality, for all $m \in [ct, t]$: $Pr(|r_{m} - R| > \epsilon) \leq 2 \exp(-\frac{m \cdot \epsilon^2}{2R_{max}^2})$. By union bound:  $Pr(\exists m \in [ct,t] : |r_m -R| > \epsilon) \leq \sum_{j=ct}^t Pr(|r_{j} - R| > \epsilon) \leq (t-ct) \cdot 2 \exp(-\frac{ct \cdot \epsilon^2}{2R_{max}^2})$.

Notice, that for a large enough $t$, we have:
$(t-ct) \cdot 2 \exp(-\frac{ct \cdot \epsilon^2}{2R_{max}^2}) \leq 2 \exp(-\frac{0.5 \cdot ct \cdot \epsilon^2}{2R_{max}^2})$ and thus we get:
$Pr(\exists m \in [ct,t] : |r_m -R| > \epsilon) \leq 2 \exp(-\frac{ct \cdot \epsilon^2}{4R_{max}^2}) $

From the law of total probability:
\begin{align*}
&Pr(|r_A - R| > \epsilon) = \sum_{j=ct}^{t} Pr(A=j) \cdot Pr(|r_A - R| > \epsilon | A = j)  \\
\end{align*}
The event $| R_{A}-R| > \epsilon | A=j$ occurs only if $\exists m \in [ct,t]: |r_m-R| > \epsilon$ hence can be upper bounded as above, and get:

\begin{align*}
Pr(|r_A - R| > \epsilon) &\leq \sum_{j=ct}^{t} Pr(A=j) \cdot 2 \exp(-\frac{j\cdot \epsilon^2}{4R_{max}^2}) \\
&\leq 2 \exp(-\frac{ct \cdot \epsilon^2}{4R_{max}^2}) \cdot \sum_{j=ct}^{t} Pr(A=j) \\
& \leq 2 \exp(-\frac{ct \cdot \epsilon^2}{4R_{max}^2})
\end{align*}

\end{proof}

\subsection{Proof of Lemma~\ref{lemma:NoiseTermConvergenceTime}}
\label{app:noiseTermConvergenceTimeProof}
\begin{proof}
Taking $\delta_1 = \frac{\delta}{N}$ where $N = \frac{2}{1-\gamma} log(\frac{D_0}{\epsilon_1})$ and $D_0$ is a constant, we have from equation~\ref{eq:t_k} in the proof (appendix~\ref{app:convergenceRateProof}) that if $t_k \geq \frac{4287 \cdot |S| |A| \cdot \max(M,K) (4R_{max} + 4 \gamma V_{max})^2 \cdot \log(\frac{|S||A|}{\delta_1})}{c \cdot \epsilon^2 D_k^2 (M+K)} + 24$, then:


\begin{align*}
&Pr(\forall t \in [t_{k+1}, t_{k+2}] : |W_{t:t_k}| \leq \frac{\epsilon}{3}D_k ) \geq 1 - \delta_1
\end{align*}

Taking the union bound over all bad events, we have that:

\begin{align*}
&Pr(\forall k \in [0,N], t \in [t_{k+1}, t_{k+2}] : |W_{t_k:t}| \leq \frac{\epsilon}{3}D_k) \\
&= 1 - Pr(\exists k \in [0,N], t \in [t_{k+1}, t_{k+2}] : |W_{t_k:t}| \geq \frac{\epsilon}{3}D_k) \\
& \geq 1 - \sum_{k=0}^N Pr(\exists t \in [t_{k+1}, t_{k+2}] : |W_{t_k:t}| \geq \frac{\epsilon}{3}D_k) \\
&\geq 1 - \sum_{k=0}^N \delta_1 \geq 1 - \delta
\end{align*}

Given that $t_0 \geq \frac{4287 \cdot |S| |A| \cdot \max(M,K) (4R_{max} + 4 \gamma V_{max})^2 \cdot \log(\frac{|S||A|}{\delta_1})}{c \cdot \epsilon^2 D_{N-1}^2 (M+K)} + 24$. 
Since $D_{N-1} > \epsilon_1$, we can take:
$t_0 \geq \frac{4287 \cdot |S| |A| \cdot \max(M,K) (4R_{max} + 4 \gamma V_{max})^2 \cdot \log(\frac{|S||A| N}{\delta})}{c \cdot \epsilon^2 \epsilon_1^2 (M+K)} + 24$

\end{proof}

\subsection{Proof of Corollary~\ref{corollary:convergenceSingleUpdatePerStep}}
\label{sec:corollaryProof}
Here, we are considering the standard asynchronous setting, where at every iteration a single state-action pair is updated. The explicit algorithm is specified in appendix~\ref{app:alg}. The proof in the asynchronous case is very similar to the synchronous case, except that we need to consider iterations in which we updated each state-action pair, and make sure that we update each pair frequently enough to get convergence. This is achieved via assumption~\ref{assum:ConditionsForConvergence}. 

The main idea here, is that we have a longer 'cycle' length until all state-action pairs are updated, compared to a single iteration as in the synchronous case. We use again assumption~\ref{assum:ConditionsForConvergence} in the following way - we notice that the 'cycle' time until all pairs are updated is $\frac{1}{c}|S||A|$ iterations, since we are ensured to have at every time point at least $c$ fraction of samples of each $(s,a)$. This allows us to use a modification of the previous proof and gives corollary~\ref{corollary:convergenceSingleUpdatePerStep}. 
This happens since at every time point $t > T_0$ we are ensured to have at least $\frac{c \cdot t}{|S| \cdot |A|}$ for every state-action pair, thus at iteration $t' = t + \frac{1}{c} |S| |A|$, we have at least $\frac{c \cdot t}{|S| \cdot |A|}+1$ samples of each pair in the memory buffer. This means, for each state-action pair, that within $\frac{1}{c} |S| |A|$ iterations, we either updated that pair at least once, or otherwise we already had at least $\frac{c \cdot t}{|S| \cdot |A|}+1$ samples of this pair already at time $t$, so that assumption~\ref{assum:ConditionsForConvergence} will still hold at time $t'$, i.e, we had at least $\frac{c \cdot t}{|S| \cdot |A|}+1$ samples in the memory buffer already at time $t$. In either case, at iteration $t'$, we performed at least $\frac{c \cdot t}{|S| \cdot |A|}+1$ updates on that pair. This allows us to analyze the convergence rate of Q-learning with experience replay when only one state-action pair is being updated at every step, by looking at a cycle length of $\frac{1}{c}|S||A|$ as the time in which all state-action pairs have been updated, instead of a cycle length of a single iteration as in the synchronous case in Theorem~\ref{theorem:convergenceRate}, where every state-action pair is been updated at every iteration. We note here that we use a learning rate $\alpha_t(s,a) = \frac{1}{n(s,a)}$ where $n(s,a)$ is the number of updates done on $Q(s,a)$ until time $t$ for the state-action pair that is updated at iteration $t$, and $\alpha_t(s,a) = 0$ for the rest. Essentially, we are only interested in iterations where $\alpha_t(s,a) \neq 0$, since only in these iterations do we change the Q values. Considering only these iterations, we have a convergence process which is similar to the process the occurs when we update all pairs in every iteration.
We now turn to specify the changes required in the proof in order to get the bound in the asynchronous case.

As before, we use the framework of iterative stochastic algorithms, starting by showing how Q-learning with experience replay can be written as an iterative stochastic algorithm. In order to avoid repetition, we refer the reader to the proof of Theorem~\ref{theorem:convergenceRate} in appendix~\ref{app:convergenceRateProof}. As before, we can write the difference between our Q values and the optimal Q values, $Q_t(s,a) - Q^\ast(s,a)$, as a sum of two components - $Y_t(s,a)$, the 'deterministic' error, and $W_t(s,a)$, the 'stochastic' error. We bound each separately. First, we note that for each $(s,a)$, by looking only at iterations where $\alpha_t(s,a) \neq 0$, and applying lemma~\ref{lemma:convergenceInLimit} to these iterations (which are the only iterations where the Q value of $(s,a)$ change), we get the same result, i.e., we have:
\convergenceInLimit*

Furthermore, in the proof of lemma~\ref{lemma:convergenceInLimit} we show that for every $s,a, t \geq t_k$, the Q values are bounded:
\begin{align*}
-Y_t(s,a) - W_{t:t_k}(s,a) \leq Q_{t}(s,a) - Q^\ast(s,a) \leq Y_t(s,a) + W_{t:t_k}(s,a)
\end{align*}
Where we define $\forall \tau,s,a: W_{\tau:\tau}(s,a) = 0$ and $W_{t+1:\tau}(s,a) = (1-\alpha_t(s,a))W_{t:\tau}(s,a) + \alpha_t(s,a)w_t(s,a)$.
Additionally, let $Y_{t_k}(s,a) = D_k$ and $\forall t \geq t_k : Y_{t+1}(s,a) = (1-\alpha_t(s,a))Y_t(s,a) + \alpha_t(s,a) \gamma D_k$.

Thus, we use this decomposition, starting by bounding the number of iteration for $Y_t(s,a)$ to become smaller than $(\gamma + \frac{2 \epsilon}{3}) D_k$, then bounding the number of iterations until $W_t(s,a)$ is smaller than $\frac{\epsilon}{3} D_k$.
We start with $Y_t$. We take here:
$t_{k+1} = \frac{|S| |A|}{c}\cdot 3 \cdot t_k$, which assures that all state-action pairs are updated enough times in every interval in which we shrink the distance between our $Q$ values and the optimal ones (i.e., in every interval between being bounded by $D_k$ and being bounded by $D_{k+1}$), which stems from assumption~\ref{assum:ConditionsForConvergence}. This means that within $\frac{|S| |A|}{c}\cdot 3 \cdot t_k$ iterations, there were at least $3 \cdot t_k$ updates for every state-action pair, giving us immediately the result of Lemma~\ref{lemma:deterministicPartConvergenceTime}, only with $t_{k+1} = \frac{|S| |A|}{c}\cdot 3 \cdot t_k$, i.e. we have:
\begin{restatable}{lemma}{deterministicPartConvergenceTime_asynchronous}
\label{lemma:deterministicPartConvergenceTime_asynchronous}
Let $\gamma \in (0,1)$, $ \epsilon = \frac{1 - \gamma}{2}$, and assume at time $t_k$ we have $Y_{t_k}(s,a) = D_k$. Then, for all $t > 3 \cdot \frac{|S| |A|}{c} \cdot t_k$ we have $Y_t(s,a) < (\gamma + \frac{2\epsilon}{3})D_k$
\end{restatable}

Thus we are left to prove the convergence of $W_t$.
We have, as in the proof of Theorem~\ref{theorem:convergenceRate}, that we can write every $W_t$ as an average of samples from the MDP anf from the memory buffer. As before,we bound the probability that for a time point $m \in [t_{k+1}, t_{k+2}]$, $W_m(s,a)$ is large, i.e.: $\Pr(|W_m(s,a)| \geq \frac{\epsilon}{3}D_k)$.
First we note that since we are using $\alpha_t = \frac{1}{n(s,a)}$, and looking only at the time points where the pair $(s,a)$ was updated, i.e. points in which $\alpha_t(s,a) \neq 0$, we again have that $W_m$ is in fact an average of samples $w_i$ for $i\in \{1,\ldots,m\}$, where some are online samples from the MDP, and some are samples from the memory buffer. In oints in time where $(s,a)$ was not updated, we have $w_t=0$ and thus no noise is accumulated in those iterations. First, note that we have at most $m$ non-zero terms in $W_m$. Now, we denote online samples as $q_i$, and replay samples as $g_i$. 
Since we are taking $M$ replay iterations every $K$ online iterations, we have that out of the $m$ samples in $W_m$, at most $\frac{mM}{M+K}$ of then are replay samples, and $\frac{mK}{M+K}$ of them are samples from the MDP. Since replay samples incur a higher noise than online samples, we can look at the worst case in order to bound the probability that $W_m$ is large, and look at the sum as if all $m$ iterations have non-zero terms, and out of them $\frac{mM}{M+K}$ are replay samples.
We will omit in the following analysis the state-action pair, and refer to a single and constant pair.  
We write $W_m$ as two sums, each summing over samples of one kind:
\begin{align*}
\Pr\left(|W_m| \geq \frac{\epsilon}{3}D_k\right) &= \Pr\left(\left| \frac{1}{t_k+m}\left(\sum_{i=1}^{\frac{mM}{M+K}}g_i + \sum_{i=1}^{\frac{mK}{M+K}}q_i \right) \right| \geq \frac{\epsilon}{3} D_k \right) \\
& \leq \Pr \left( \left| \frac{1}{t_k+m} \sum_{i=1}^{\frac{mM}{M+K}}g_i \right| \geq \frac{\epsilon}{6}D_k \right) + \Pr \left( \left| \frac{1}{t_k+m} \sum_{i=1}^{\frac{mK}{M+K}}q_i \right| \geq \frac{\epsilon}{6}D_k \right)
\end{align*}
As specified in the proof of Theorem~\ref{theorem:convergenceRate}, we bound the different types of iterations separately (this is exactly the same as in the synchronous case, thus we only state the final bound we get). We get:
\begin{align*}
&Pr(|W_m | \geq \frac{\epsilon}{3}D_k) \\
&\leq 3 \exp\left(\frac{-\epsilon^2 D_k^2 (m+t_k)^2 (M+K)}{2304 |S| |A| m \cdot \max(M,K) (4R_{max}+4 \gamma V_{max})^2}\right) \\
&+ 4 \cdot \frac{m M}{M+K} \cdot \exp\left( \frac{-c \cdot t_k \epsilon^2 D_k^2 (m+t_k)^2 (M+K)}{|S| |A| 2304 m^2 \cdot \max(M,K) (4R_{max} + 4 \gamma V_{max})^2} \right) 
\end{align*}
We notice that here, $m \in [t_{k+1}, t_{k+2}]$ means that: $ 3.5 \frac{|S||A|}{c} \cdot t_k \leq m \leq 15.75 \frac{(|S||A|)^2}{c^2} \cdot t_k + 3.5 \frac{|S||A|}{c} \cdot t_k$, and so we get:
\begin{align*}
&Pr(|W_m | \geq \frac{\epsilon}{3}D_k) \\
&\leq 3 \exp\left(\frac{-\epsilon^2 D_k^2 (4.5 \frac{|S||A|}{c} \cdot t_k)^2 (M+K)}{2304 |S| |A| 3.5 \frac{|S||A|}{c} \cdot t_k \cdot \max(M,K) (4R_{max}+4 \gamma V_{max})^2}\right) \\
&+ 4 \cdot \frac{3.5 \frac{|S||A|}{c} \cdot t_k M}{M+K} \cdot \exp\left( \frac{-c \cdot t_k \epsilon^2 D_k^2 ( 4.5 \frac{|S||A|}{c} \cdot t_k)^2 (M+K)}{|S| |A| 2304 (3.5 \frac{|S||A|}{c} \cdot t_k)^2 \cdot \max(M,K) (4R_{max} + 4 \gamma V_{max})^2} \right) \\
&\leq 3 \exp\left(\frac{ -0.008 \frac{|S||A|}{c} \epsilon^2 D_k^2 t_k (M+K)}{|S| |A| \cdot \max(M,K) (4R_{max}+4 \gamma V_{max})^2}\right) \\
&+  14 \cdot \frac{|S||A|}{c} \frac{ t_k M}{M+K} \cdot \exp\left( \frac{-0.0004 \cdot c \cdot t_k \epsilon^2 D_k^2 (M+K)}{|S| |A| \cdot \max(M,K) (4R_{max} + 4 \gamma V_{max})^2} \right) \\
& \leq \left(14 \cdot \frac{|S||A|}{c} \frac{t_k M}{M+K} + 3 \right) \cdot \exp\left( \frac{-0.0004 \cdot c \cdot t_k \epsilon^2 D_k^2 (M+K)}{|S| |A| \cdot \max(M,K) (4R_{max} + 4 \gamma V_{max})^2} \right) \\
& = \exp\left( \frac{-0.0004 \cdot c \cdot t_k \epsilon^2 D_k^2 (M+K)}{|S| |A| \cdot \max(M,K) (4R_{max} + 4 \gamma V_{max})^2} + log\left(14 \cdot \frac{|S||A|}{c} \frac{t_k M}{M+K} + 3 \right) \right)
\end{align*}

We would like to have that for all $m \in [t_{k+1}, t_{k+2}]: |W_m| \leq \frac{\epsilon}{3} D_k$ with high probability (defined by the user). Here we take $t_{k+1}$ since this is when $Y_t(s,a)$ becomes small enough. Call this probability $1 - \delta$ for some $\delta \in (0,1)$. We use the union bound to get:

\begin{align*}
&Pr(\forall m \in [t_{k+1}, t_{k+2}] : |W_m| \leq \frac{\epsilon}{3}D_k ) = 1 - Pr(\exists m \in [t_{k+1}, t_{k+2}] : |W_m| \leq \frac{\epsilon}{3}D_k ) \\
& \geq 1 - \sum_{i=t_k}^{t_{k+2}} Pr(|W_i| \geq \frac{\epsilon}{3} D_k)
\end{align*}

If we want this to happen with probability $1 - \delta$, for all pairs of states and actions, then we need: 

\begin{align*}
Pr(|W_m| \geq \frac{\epsilon}{3} D_k) \leq \frac{\delta}{|S| |A| \cdot 12.25 \frac{(|S||A|)^2}{c^2} t_k}
\end{align*}

For this to take place we need to have:
\begin{align*}
&\exp\left( \frac{-0.0004 \cdot c \cdot t_k \epsilon^2 D_k^2 (M+K)}{|S| |A| \cdot \max(M,K) (4R_{max} + 4 \gamma V_{max})^2} + log\left(14 \cdot \frac{|S||A|}{c} \frac{t_k M}{M+K} + 3 \right) \right) \\
&\leq \frac{\delta}{(|S| |A|)^3 \cdot 12.25 \frac{1}{c^2} \cdot t_k} 
\end{align*}
which means:
\begin{align*}
&\exp\left( \frac{-0.0004 c t_k \epsilon^2 D_k^2 (M+K)}{|S| |A| \max(M,K) (4R_{max} + 4 \gamma V_{max})^2} + log\left(14 \frac{|S||A|}{c} \frac{t_k M}{M+K} + 3 \right) + \log \left( 12.25 \frac{1}{c^2} t_k\right) \right) \\
&\leq \frac{\delta}{(|S| |A|)^3} 
\end{align*}
And so:
\begin{align*}
& \frac{-0.0004 \cdot c \cdot t_k \epsilon^2 D_k^2 (M+K)}{|S| |A| \cdot \max(M,K) (4R_{max} + 4 \gamma V_{max})^2} + log\left(14 \cdot \frac{|S||A|}{c}  \frac{t_k M}{M+K} + 3 \right) + \log \left( 12.25 \frac{1}{c^2} t_k \right) \\
&\leq \log(\frac{\delta}{(|S| |A|)^3}) \\ 
\end{align*}

Thus we need: \\
$t_k - log\left(14 \cdot \frac{|S||A|}{c} \frac{t_k M}{M+K} + 3 \right) - \log \left( 12.25 \frac{1}{c^2} t_k \right) \geq \frac{2500 \cdot |S| |A| \cdot \max(M,K) (4R_{max} + 4 \gamma V_{max})^2 \cdot \log(\frac{(|S||A|)^3}{\delta})}{c \cdot \epsilon^2 D_k^2 (M+K)}$

We use the fact that $log(x) < \sqrt{x}$, and we upper bound this:

\begin{align*}
&t_k - log\left(14 \cdot \frac{|S||A|}{c} \frac{t_k M}{M+K} + 3 \right) - \log \left( 12.25 \frac{1}{c^2} t_k \right) \\
&\geq t_k - log\left((14 \cdot \frac{|S||A|}{c} \frac{M}{M+K} + 3) t_k \right) - \log \left( 12.25 \frac{1}{c^2} t_k \right) \\
& \geq t_k - log\left(12.25 \frac{1}{c^2} \cdot (14 \cdot \frac{|S||A|}{c} \frac{M}{M+K} + 3) t_k^2 \right) \\
& \geq t_k - \log(12.25 \frac{1}{c^2} \cdot (14 \cdot \frac{|S||A|}{c} \frac{M}{M+K} + 3)) - 2\sqrt{t_k}
\end{align*}

We need: \\
$t_k - \log(12.25 \frac{1}{c^2}\cdot (14 \cdot \frac{|S||A|}{c} \frac{M}{M+K} + 3)) - 2\sqrt{t_k} \geq \frac{2500 \cdot |S| |A| \cdot \max(M,K) (4R_{max} + 4 \gamma V_{max})^2 \cdot \log(\frac{(|S||A|)^3}{\delta})}{c \cdot \epsilon^2 D_k^2 (M+K)}$

This happens when:
\begin{align*}
&t_k - 2\sqrt{t_k} \\
&\geq \frac{2500  |S| |A| \cdot \max(M,K) (4R_{max} + 4 \gamma V_{max})^2 \cdot \log(\frac{(|S||A|)^3}{\delta})}{c \cdot \epsilon^2 D_k^2 (M+K)} + \log(172 \frac{|S||A|}{c^3}  \frac{M}{M+K} + 37 \frac{1}{c^2})
\end{align*}

Thus, we need:
\begin{align*}
t_k &\geq \frac{2500  |S| |A| \max(M,K) (4R_{max} + 4 \gamma V_{max})^2 \log(\frac{|S||A|}{\delta})}{c \cdot \epsilon^2 D_k^2 (M+K)} + \log(172 \frac{|S||A|}{c^3}  \frac{M}{M+K} + 37 \frac{1}{c^2}) + 2 \\
&+ 2 \sqrt{\frac{2500  |S| |A| \max(M,K) (4R_{max} + 4 \gamma V_{max})^2 \log(\frac{|S||A|}{\delta})}{c \cdot \epsilon^2 D_k^2 (M+K)} + \log(172 \frac{|S||A|}{c^3}  \frac{M}{M+K} + 37 \frac{1}{c^2}) + 1} \\
\end{align*}

It is enough to take:
\begin{align*}
t_k = \Omega \left( \frac{ |S| |A| \cdot \max(M,K) (R_{max} + \gamma V_{max})^2 \cdot \log(\frac{(|S||A|)^3}{\delta})}{c \cdot \epsilon^2 D_k^2 (M+K)} + log(\frac{|S||A|}{c^3}) \right)
\end{align*}

Using this $t_k$ in Lemma~\ref{lemma:NoiseTermConvergenceTime}, we get that if: \\

$T= \Omega \left( \left(3 \frac{|S||A|}{c} \right)^{ \frac{2}{1-\gamma}log(\frac{D_0}{\epsilon_1})} \left( \frac{ |S| |A| \max(M,K) (R_{max} + \gamma V_{max})^2 \left(\log(\frac{(|S||A|)^3}{\delta})+\log( \frac{2}{1-\gamma} log(\frac{D_0}{\epsilon_1}))\right) }{c \cdot (1-\gamma)^2 \epsilon_1^2 (M+K)} + log(\frac{|S||A|}{c^3}) \right) \right)$, then $\forall t \geq T: \|Q - Q^\ast\|_\infty \leq \epsilon_1$ w.p. at least $1-\delta$.

\subsection{Relaxation of Assumption~\ref{assum:ConditionsForConvergence}}
\label{app:assumption_1_relaxation}
We assume the following assumption, a probabilistic version of assumption~\ref{assum:ConditionsForConvergence}:
\begin{assumption}
\label{ass:probabilistic_assumption_converegence}
There exists some constant $c \in (0,1)$ s.t. with probability $\frac{1}{2}$, from any initial state $s$, within $\frac{|S||A|}{c}$ iterations, all state-action pairs are visited.
\end{assumption}
Given this assumption, we note the following:
Given $\delta_1 \in (0,1)$, let $k = \log_2(\frac{1}{\delta_1})$.
Then, the probability of not vising all state-action pairs within $k \cdot \frac{|S||A|}{c}$ iterations is $\leq 2^{-k} \leq \delta_1$. Thus, with probability at least $1 - \delta_1$, from any start state $s \in S$, after running for $k \cdot \frac{|S||A|}{c}$, all state-action pairs are visited.
Next, for this to hold throughout the run of the algorithm, we will use a union bound argument. For a run length of $T$ iterations, choose $k=\log_2(\frac{T}{\delta_1})$, then, we have the following:
\begin{align*}
&\Pr( \mbox{not all state-action pairs are visited within }k \cdot \frac{|S||A|}{c} \mbox{iterations, for all } t \in \{ 1,...,T\}) \\
&\leq \sum_{t = 1}^T \Pr( \mbox{not all state-action pairs are visited within }k \cdot \frac{|S||A|}{c} \mbox{iterations, starting at iteration } t) \\
& \leq \sum_{t = 1}^T 2^{-k} \leq \sum_{t = 1}^T \frac{\delta_1}{T} \leq \delta_1
\end{align*}
Thus, with probability at least $1- \delta_1$, for any $T$, within $\log_2(\frac{T}{\delta_1}) \cdot \frac{|S||A|}{c}$ iterations, all state-action pairs are visited, for all $t \in \{1,...,T\}$.

Now, with the above high-probability bound, we can formulate the convergence rate when assuming assumption~\ref{ass:probabilistic_assumption_converegence}.
The first change in the analysis lies in the bound on $Pr(|W_m | \geq \frac{\epsilon}{3}D_k)$. Here, we cannot assume that at time $t_k$ we have at least $\frac{c \cdot t_k}{|S||A|}$ samples of each $(s,a)$, rather, we use the following analysis: at time point $t_k$ we set $k=\log_2(\frac{T}{\delta_2})$ where $\delta_2 = \frac{\delta}{2}$ for $\delta \in (0,1)$.
Then, we have that with probability at least $1-\delta_2$, we have at least $\frac{c \cdot t_k}{|S||A| \log_2{\frac{T}{\delta_2}}}$ samples of each $(s,a)$.
Thus we get, for $m \in [t_{k+1}, t_{k+2}]$:
\begin{align*}
&Pr(|W_m | \geq \frac{\epsilon}{3}D_k) \leq \exp\left( \frac{-0.0007 \cdot c \cdot t_k \epsilon^2 D_k^2 (M+K)}{|S| |A| \log_2(\frac{T}{\delta_2}) \cdot \max(M,K) (4R_{max} + 4 \gamma V_{max})^2} + log\left(14 \cdot \frac{t_k M}{M+K} + 3 \right) \right)
\end{align*}

We have bounded the probability for a single $m \in [t_{k+1},t_{k+2}]$ to be very large. Now, we would like to have that for all $m \in [t_{k+1}, t_{k+2}]: |W_m| \leq \frac{\epsilon}{3} D_k$ with high probability $1-\delta_2$. Here we take $t_{k+1}$ as the starting point in time where we bound $W_m$ since this is when $Y_t(s,a)$ becomes small enough. We then ensure that this happens for all $k$, which will give us the bound for all time points. We use the union bound to get:
\begin{align*}
Pr(\forall m \in [t_{k+1}, t_{k+2}] : |W_m| \leq \frac{\epsilon}{3}D_k ) &= 1 - Pr(\exists m \in [t_{k+1}, t_{k+2}] : |W_m| \geq \frac{\epsilon}{3}D_k ) \\
& \geq 1 - \sum_{i=t_k}^{t_{k+2}} Pr(|W_i| \geq \frac{\epsilon}{3} D_k)
\end{align*}
If we want this to happen with probability $1 - \delta_1$, for all pairs of states and actions, then we need: 
\begin{align*}
Pr(|W_m| \geq \frac{\epsilon}{3} D_k) \leq \frac{\delta_1}{|S| |A| \cdot 12.25 t_k}
\end{align*}

For this to take place we need to have:
\begin{align*}
&\exp\left( \frac{-0.0007 \cdot c \cdot t_k \epsilon^2 D_k^2 (M+K)}{|S| |A| \log_2(\frac{T}{\delta_2}) \cdot \max(M,K) (4R_{max} + 4 \gamma V_{max})^2} + log\left(14 \frac{t_k M}{M+K} + 3 \right) \right) \leq \frac{\delta_1}{|S| |A| \cdot 12.25 \cdot t_k} \\
&
\end{align*}
We get:
\begin{align*}
& \frac{-0.0007 \cdot c \cdot t_k \epsilon^2 D_k^2 (M+K)}{|S| |A| \log_2(\frac{T}{\delta_2}) \cdot \max(M,K) (4R_{max} + 4 \gamma V_{max})^2} 
+ log\left(14 \cdot \frac{t_k M}{M+K} + 3 \right) + \log \left( 12.25 t_k \right) \leq \log(\frac{\delta_1}{|S| |A|}) \\
& t_k - log\left(14 \cdot \frac{t_k M}{M+K} + 3 \right) - \log \left( 12.25 t_k \right) \geq \frac{1429 \cdot |S| |A| \log_2(\frac{T}{\delta_2}) \cdot \max(M,K) (4R_{max} + 4 \gamma V_{max})^2 \cdot \log(\frac{|S||A|}{\delta_1})}{c \cdot \epsilon^2 D_k^2 (M+K)} 
\end{align*}

We use the fact that $log(x) < \sqrt{x}$, and we upper bound this:
\begin{align*}
t_k - log\left(14 \cdot \frac{t_k M}{M+K} + 3 \right) - \log \left( 12.25 t_k \right) &\geq t_k - log\left((14 \cdot \frac{M}{M+K} + 3) t_k \right) - \log \left( 12.25 t_k \right) \\
& \geq t_k - log\left(12.25 \cdot (14 \cdot \frac{M}{M+K} + 3) t_k^2 \right) \\
& \geq t_k - \log(12.25 \cdot (14 \cdot \frac{M}{M+K} + 3)) - 2\sqrt{t_k}
\end{align*}

We need:
\begin{align*}
t_k - \log(12.25 \cdot (14 \cdot \frac{M}{M+K} + 3)) - 2\sqrt{t_k} \geq \frac{1429 \cdot |S| |A| \log_2(\frac{T}{\delta_2}) \cdot \max(M,K) (4R_{max} + 4 \gamma V_{max})^2 \cdot \log(\frac{|S||A|}{\delta_1})}{c \cdot \epsilon^2 D_k^2 (M+K)} 
\end{align*}

This happens when:
\begin{align*}
t_k - 2\sqrt{t_k} \geq \frac{1429 \cdot |S| |A| \log_2(\frac{T}{\delta_2}) \cdot \max(M,K) (4R_{max} + 4 \gamma V_{max})^2 \cdot \log(\frac{|S||A|}{\delta_1})}{c \cdot \epsilon^2 D_k^2 (M+K)} + \log((172 \cdot \frac{M}{M+K} + 37))
\end{align*}


We thus get that it is enough to have:
\begin{align*}
\begin{split}
&t_k = \Omega \left( \frac{ |S| |A| \cdot \log_2(\frac{T}{\delta_2}) \max(M,K) (R_{max} + \gamma V_{max})^2  \log(\frac{|S||A|}{\delta_1})}{c \cdot \epsilon^2 D_k^2 (M+K)} \right)
\end{split}
\end{align*}

Using a similar argument as in lemma~\ref{lemma:NoiseTermConvergenceTime}, we get that if $\delta_1 = \frac{\delta}{2N}$, where $N = \frac{2}{1-\gamma} \log(\frac{D_0}{\epsilon_1})$, then with probability at least $1-
\delta$, we have that:
\begin{align*}
&Pr(\forall k \in [0,N], t \in [t_{k+1}, t_{k+2}] : |W_{t_k:t}| \leq \frac{\epsilon}{3}D_k) \geq 1-\delta
\end{align*}
given that:
\begin{align*}
\begin{split}
&t_0 = \Omega \left( \frac{ |S| |A| \log_2(\frac{T}{
\delta}) \max(M,K) (R_{max} + \gamma V_{max})^2  \log(\frac{ |S||A| N}{\delta})}{c \cdot \epsilon^2 \epsilon_1^2 (M+K)} \right)
\end{split}
\end{align*}
Thus, we need to have:
\begin{align*}
T \geq 3^{\frac{2}{1-\gamma} \log(\frac{D_0}{\epsilon_1})} \cdot \frac{ |S| |A| \log_2(\frac{T}{
\delta}) \max(M,K) (R_{max} + \gamma V_{max})^2  \log(\frac{ |S||A| N}{\delta})}{c \cdot \epsilon^2 \epsilon_1^2 (M+K)}
\end{align*}
We denote the following:
\begin{align*}
&B = 3^{\frac{2}{1-\gamma} \log(\frac{D_0}{\epsilon_1})} \cdot \left( \frac{ |S| |A| \max(M,K) (R_{max} + \gamma V_{max})^2  \log(\frac{ |S||A| N}{\delta})}{c \cdot \epsilon^2 \epsilon_1^2 (M+K)} \right) \\
&C = 3^{\frac{2}{1-\gamma} \log(\frac{D_0}{\epsilon_1})} \cdot \left( \frac{ |S| |A| \max(M,K) (R_{max} + \gamma V_{max})^2  \log(\frac{ |S||A| N}{\delta})}{c \cdot \epsilon^2 \epsilon_1^2 (M+K)} \cdot \log_2(\frac{1}{\delta}) \right)  
\end{align*}

Thus we get the following: 
Let $Q_T$ be the Q-values of synchronous Q-learning with experience replay after $T$ updates using a learning rate $\alpha_t(s,a) = \frac{1}{n(s,a)}$, where $n(s,a)$ is the number of updates done on $Q(s,a)$ until iteration $t$. Let $\epsilon_1 > 0$, $\delta \in (0,1)$, $\gamma \in (0,1)$.
For all $s \in S, a \in A$: let $|r(s,a)| \leq R_{max} < \infty$.
Then, under assumption~\ref{ass:probabilistic_assumption_converegence}, with probability at least $1 - \delta$, $\| Q_t - Q^\ast\|_{\infty} \leq \epsilon_1$ for all $t \geq T$ with: 
$T = \tilde{\Omega} \left( B^2 + \sqrt{B^4 + B^2 C} + C \right)$.

\subsection{Simplification of bound~\ref{eq:bound}}
\label{app:simplificationOfBound}
The bound we got is the following:
\begin{align*}
&Pr(|W_m | \geq \frac{\epsilon}{3}D_k) \leq 2 \exp\left(\frac{-\epsilon^2 D_k^2 (m+t_k)^2 (M+K)}{18 m K (4R_{max}+4 \gamma V_{max})^2}\right) \\
&+ \frac{m M}{M+K} \cdot \left( 4 \exp\left( \frac{-c \cdot t_k \epsilon^2 D_k^2 (m+t_k)^2 (M+K)^2}{|S| |A| 2304 m^2 M^2 (4R_{max}^2 + 4 \gamma^2 V_{max}^2)} \right) \right) \\
& + \exp \left( \frac{-\epsilon^2 D_k^2 (m+t_k)^2 (M+K)}{288 m M (4Rmax + 4 \gamma Vmax)^2} \right)
\end{align*}
We upper bound this in the following set of inequalities:
\begin{align*}
Pr(|W_m | \geq \frac{\epsilon}{3}D_k) &\leq 3 \exp\left(\frac{-\epsilon^2 D_k^2 (m+t_k)^2 (M+K)}{288 m \cdot \max(M,K) (4R_{max}+4 \gamma V_{max})^2}\right) \\
&+ 4 \cdot \frac{m M}{M+K}  \cdot \exp\left( \frac{-c \cdot t_k \epsilon^2 D_k^2 (m+t_k)^2 (M+K)^2}{|S| |A| 2304 m^2 M^2 (4R_{max}^2 + 4 \gamma^2 V_{max}^2)} \right)
\end{align*}

We take an upper bound and get:
\begin{align*}
Pr(|W_m | \geq \frac{\epsilon}{3}D_k) &\leq 3 \exp\left(\frac{-\epsilon^2 D_k^2 (m+t_k)^2 (M+K)}{2304 |S| |A| m \cdot \max(M,K) (4R_{max}+4 \gamma V_{max})^2}\right) \\
&+ 4 \cdot \frac{m M}{M+K} \cdot \exp\left( \frac{-c \cdot t_k \epsilon^2 D_k^2 (m+t_k)^2 (M+K)}{|S| |A| 2304 m^2 \cdot \max(M,K) (4R_{max} + 4 \gamma V_{max})^2} \right) 
\end{align*}

Now notice that $m \in [t_{k+1}, t_{k+2}]$, thus: $ 3 \cdot t_k \leq m \leq 12 \cdot t_k$, and so we get:

\begin{align*}
&Pr(|W_m | \geq \frac{\epsilon}{3}D_k) \\
&\leq 3 \exp\left(\frac{-\epsilon^2 D_k^2 (4 \cdot t_k)^2 (M+K)}{2304 |S| |A| 3 \cdot t_k \cdot \max(M,K) (4R_{max}+4 \gamma V_{max})^2}\right) \\
&+ 4 \cdot \frac{3 \cdot t_k M}{M+K} \cdot \exp\left( \frac{-c \cdot t_k \epsilon^2 D_k^2 (4 \cdot t_k)^2 (M+K)}{|S| |A| 2304 (3 \cdot t_k)^2 \cdot \max(M,K) (4R_{max} + 4 \gamma V_{max})^2} \right) \\
&\leq 3 \exp\left(\frac{-0.0023 \cdot \epsilon^2 D_k^2 t_k (M+K)}{|S| |A| \cdot \max(M,K) (4R_{max}+4 \gamma V_{max})^2}\right) \\
&+  12 \cdot \frac{t_k M}{M+K} \cdot \exp\left( \frac{-0.0007 \cdot c \cdot t_k \epsilon^2 D_k^2 (M+K)}{|S| |A| \cdot \max(M,K) (4R_{max} + 4 \gamma V_{max})^2} \right) \\
& \leq \left(12 \cdot \frac{t_k M}{M+K} + 3 \right) \cdot \exp\left( \frac{-0.0007 \cdot c \cdot t_k \epsilon^2 D_k^2 (M+K)}{|S| |A| \cdot \max(M,K) (4R_{max} + 4 \gamma V_{max})^2} \right) \\
& = \exp\left( \frac{-0.0007 \cdot c \cdot t_k \epsilon^2 D_k^2 (M+K)}{|S| |A| \cdot \max(M,K) (4R_{max} + 4 \gamma V_{max})^2} + log\left(12 \cdot \frac{t_k M}{M+K} + 3 \right) \right)
\end{align*}

\subsection{Bound on c}
\label{app:bound_on_c}
Denote by $c' \in (0,1)$ the constant which describes the rate of covering state-action pairs in the MDP, i.e. within $\frac{|S||A|}{c'}$ iterations of samples from the MDP we have visited every state-action pair at least once. This means that at time $t$, the number of samples of a state-action pair $(s,a)$ is at least $\frac{c \cdot t}{|S||A|}$. We also know that at time $t$, only $t \cdot \frac{K}{M+K}$ of the iterations were iterations of samples from the MDP (while the rest were samples from the memory buffer). 
We also know that at time $t$ there were $t \cdot \frac{K}{M+K}$ iterations of samples from the MDP, since we perform $M$ iterations of replay every $K$ iterations. Thus, the number of iterations of samples of $(s,a)$ is $\frac{t \cdot \frac{K}{M+K}}{\frac{|S||A|}{c'}} = \frac{c' \cdot t \cdot \frac{K}{M+K}}{|S||A|}$. 
Note the in assumption~\ref{assum:ConditionsForConvergence}, we have that the covering time $\frac{|S||A|}{c}$ considers the algorithm with experience replay, i.e. where we perform $M$ iterations of replay every $K$ iterations in the MDP, we get that:
\begin{align*}
\frac{c \cdot t}{|S||A|} = \frac{c' \cdot t \cdot \frac{K}{M+K}}{|S||A|}
\end{align*}
and thus, since $c' < 1$, we get that $c \leq \frac{K}{M+K}$.

\end{document}